\documentclass{article}

\usepackage{microtype}
\usepackage{graphicx}
\usepackage{booktabs} 

\usepackage{hyperref}
\usepackage{multicol}
\usepackage{multirow}

\usepackage{subcaption}
\usepackage[most]{tcolorbox}

\usepackage{xspace}

\newcommand{\projectname}[0]{\textsc{Calibration Aware Token-level Training Objective}\xspace}

\newcommand{\projectnameshort}[0]{\textsc{CATTO}\xspace}

\usepackage[preprint]{icml2026}

\usepackage{amsmath}
\usepackage{amssymb}
\usepackage{mathtools}
\usepackage{amsthm}
\usepackage{enumitem}

\usepackage[capitalize,noabbrev]{cleveref}

\theoremstyle{plain}
\newtheorem{theorem}{Theorem}[section]
\newtheorem{proposition}[theorem]{Proposition}
\newtheorem{lemma}[theorem]{Lemma}

\theoremstyle{definition}
\newtheorem{definition}[theorem]{Definition}

\theoremstyle{remark}

\usepackage[textsize=tiny]{todonotes}

\icmltitlerunning{\projectnameshort: Balancing Preferences and Confidence in Language Models}

\begin{document}

\twocolumn[
\icmltitle{\projectnameshort: Balancing Preferences and Confidence in Language Models}





\begin{icmlauthorlist}
\icmlauthor{Nisarg Parikh}{umass}
\icmlauthor{Ananya Sai}{adobe}
\icmlauthor{Pannaga Shivaswamy}{adobe}
\icmlauthor{Kunjal Panchal}{umass}
\icmlauthor{Andrew Lan}{umass}
\end{icmlauthorlist}

\icmlaffiliation{umass}{Department of Computer Science, University of Massachussets, Amherst, USA}
\icmlaffiliation{adobe}{Adobe, Bangalore, India}

\icmlcorrespondingauthor{Nisarg Parikh}{nkparikh@umass.edu}

\icmlkeywords{Machine Learning, ICML}

\vskip 0.3in
]



\printAffiliationsAndNotice{} 

\begin{abstract}
    Large language models (LLMs) often make accurate next token predictions but their confidence in these predictions can be poorly calibrated: high-confidence predictions are frequently wrong, and low-confidence predictions may be correct.
This miscalibration is exacerbated by preference-based alignment methods breaking the link between predictive probability and correctness.
We introduce a \projectname (\projectnameshort), a calibration-aware objective that aligns predicted confidence with empirical prediction correctness, which can be combined with the original preference optimization objectives.
Empirically, \projectnameshort reduces Expected Calibration Error (ECE) by $2.22\%$-$7.61\%$ in-distribution and $1.46\%$-$10.44\%$ out-of-distribution compared to direct preference optimization (DPO), and by $0.22\%$-$1.24\%$ in-distribution and $1.23\%$-$5.07\%$ out-of-distribution compared to the strongest DPO baseline. 
This improvement in confidence does not come at a cost of losing task accuracy, where \projectnameshort maintains or slightly improves multiple-choice question-answering accuracy on five datasets.  
We also introduce Confidence@k, a test-time scaling mechanism leveraging calibrated token probabilities for Bayes-optimal selection of output tokens.
\end{abstract}

\section{Introduction}
\label{sec:introduction}
Large language models (LLMs) are increasingly deployed in settings where they must not only produce accurate outputs, but also provide reliable \textit{confidence estimates}, i.e., a probability that the prediction is correct. 
LLMs with well-calibrated confidence estimates are critical for downstream decision-making, including selective prediction, answer reranking, reasoning, and test-time scaling~\cite{geifman2019selectivenet,kadavath2022language,wang2023selfconsistency}.

Formally, a model is \textit{well-calibrated} if predictions made with confidence $p$ are empirically correct at the same frequency, $p$~\cite{guo2017calibration}. 
Despite strong gains in accuracy and generalization~\cite{ouyang2022training,chung2024scaling,yin2023llm}, modern LLMs are often severely miscalibrated, particularly after post-training and preference alignment, such as reinforcement learning from human feedback (RLHF)~\cite{christiano2017deep} and direct preference optimization (DPO)~\cite{rafailov2023direct}. 
This miscalibration causes models to be overly confident in incorrect predictions and underconfident in correct ones, as reflected by the gap between confidence and accuracy in Figure~\ref{fig:ece-plot-main}.

Prior work performs confidence calibration in LLMs both during and after training~\cite{xiao2025restoring,huang2025calibrated}, using techniques such as temperature scaling~\cite{guo2017calibration,lamb2025semantic} and label smoothing~\cite{huang2025calibrated}.
However, a major challenge in confidence calibration is that pre-training calibration does not persist after post-training preference alignment.
Preference-based objectives typically focus only on the relative likelihood ratio between preferred and unpreferred outputs, without any constrains on the absolute scale of token-level probabilities. 
Due to this, alignment such as DPO or RLHF can substantially alter logit magnitudes, leading to systematic confidence drift and poor calibration~\cite{leng2025taming,parikh2025lookalike,xiao2025algorithmic,xiao2024caldpo}. 
Empirically, this shows up as inflated logits, overconfident incorrect predictions, and underconfident correct ones (Table~\ref{tbl:in-distribution-main}, Figure~\ref{fig:ece-plot-main}), undermining the reliability of confidence estimates.

\begin{figure*}[hbt!]
    \centering
    \begin{subfigure}[b]{0.24\textwidth}
         \centering
         \includegraphics[width=\textwidth]{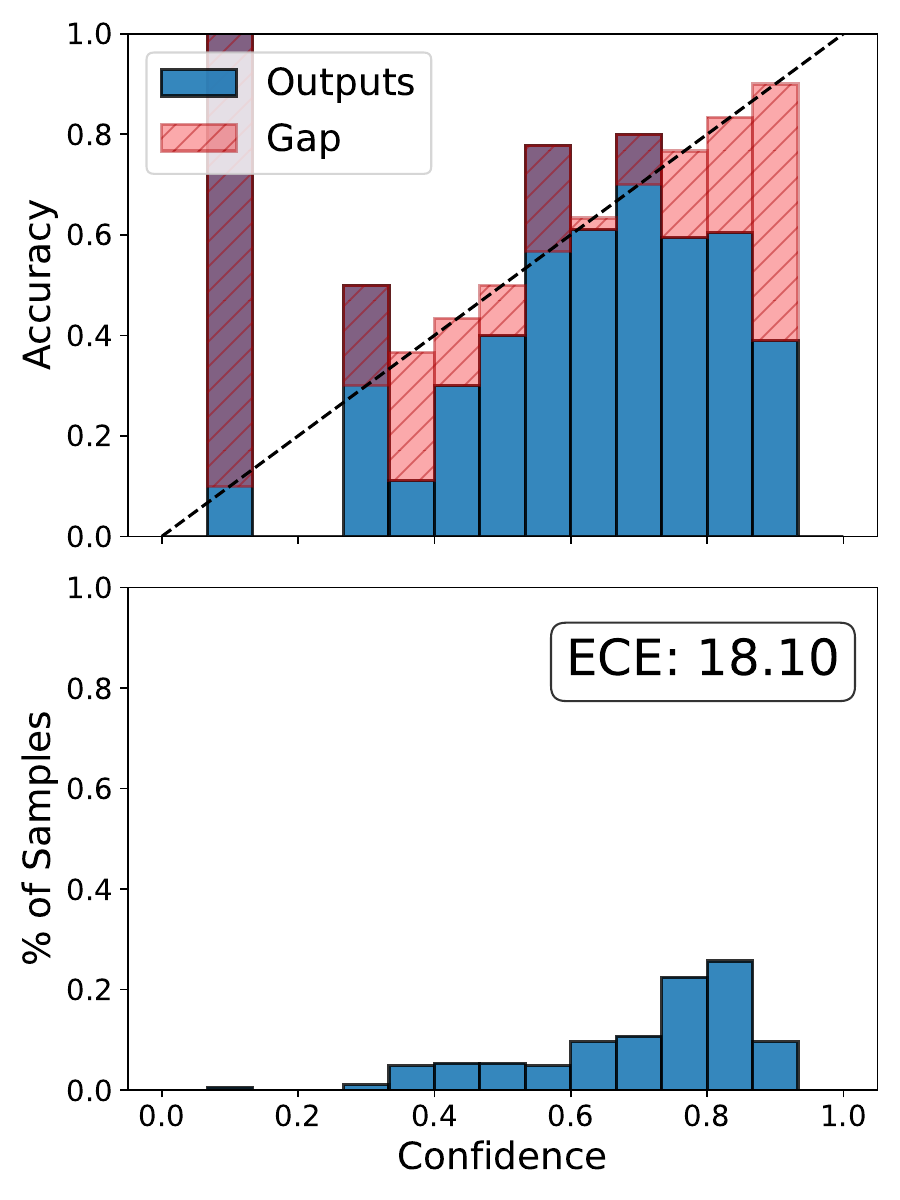}
         \caption{SFT} \label{subfig:realibility-confidence-sft}
     \end{subfigure}
     \hfill
    \begin{subfigure}[b]{0.24\textwidth}
         \centering
         \includegraphics[width=\textwidth]{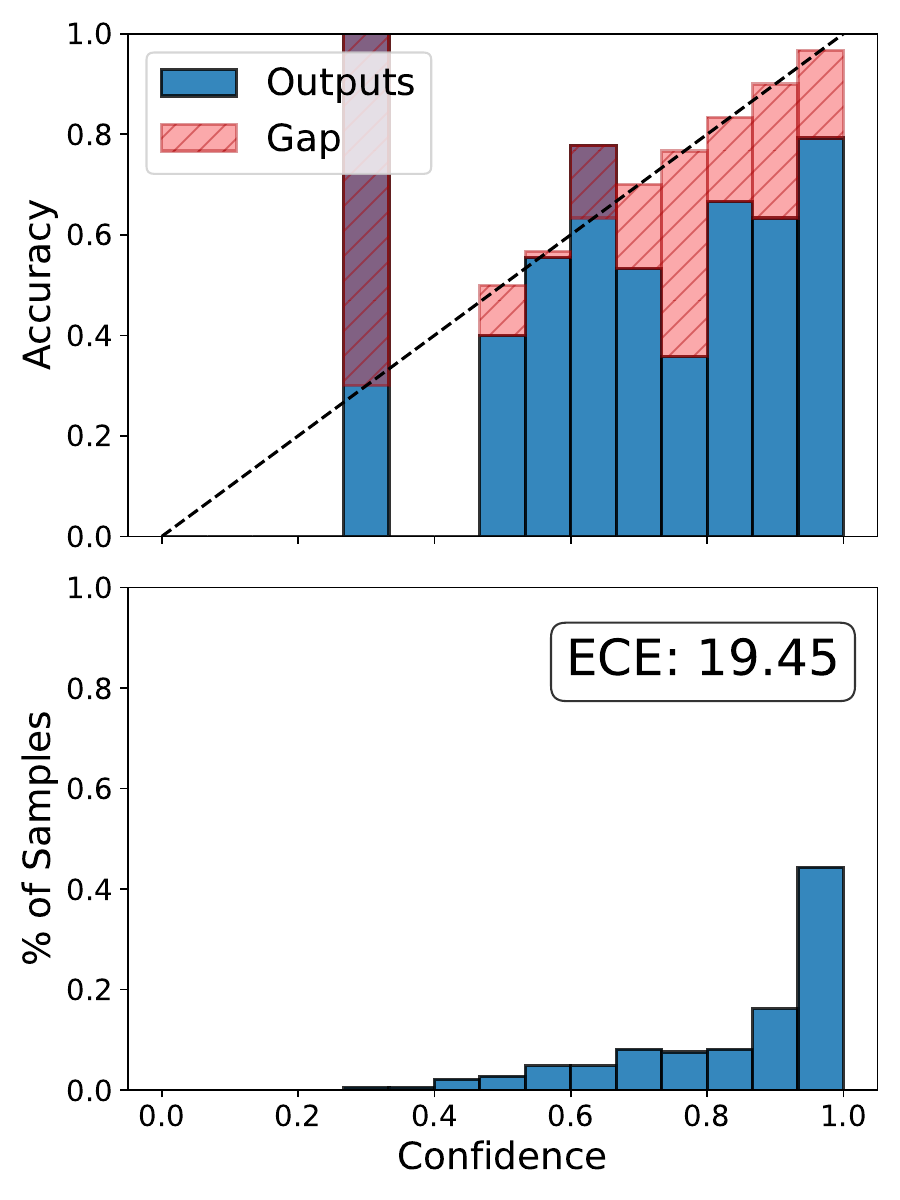}
         \caption{DPO} \label{subfig:realibility-confidence-dpo}
     \end{subfigure}
     \hfill
    \begin{subfigure}[b]{0.24\textwidth}
         \centering
         \includegraphics[width=\textwidth]{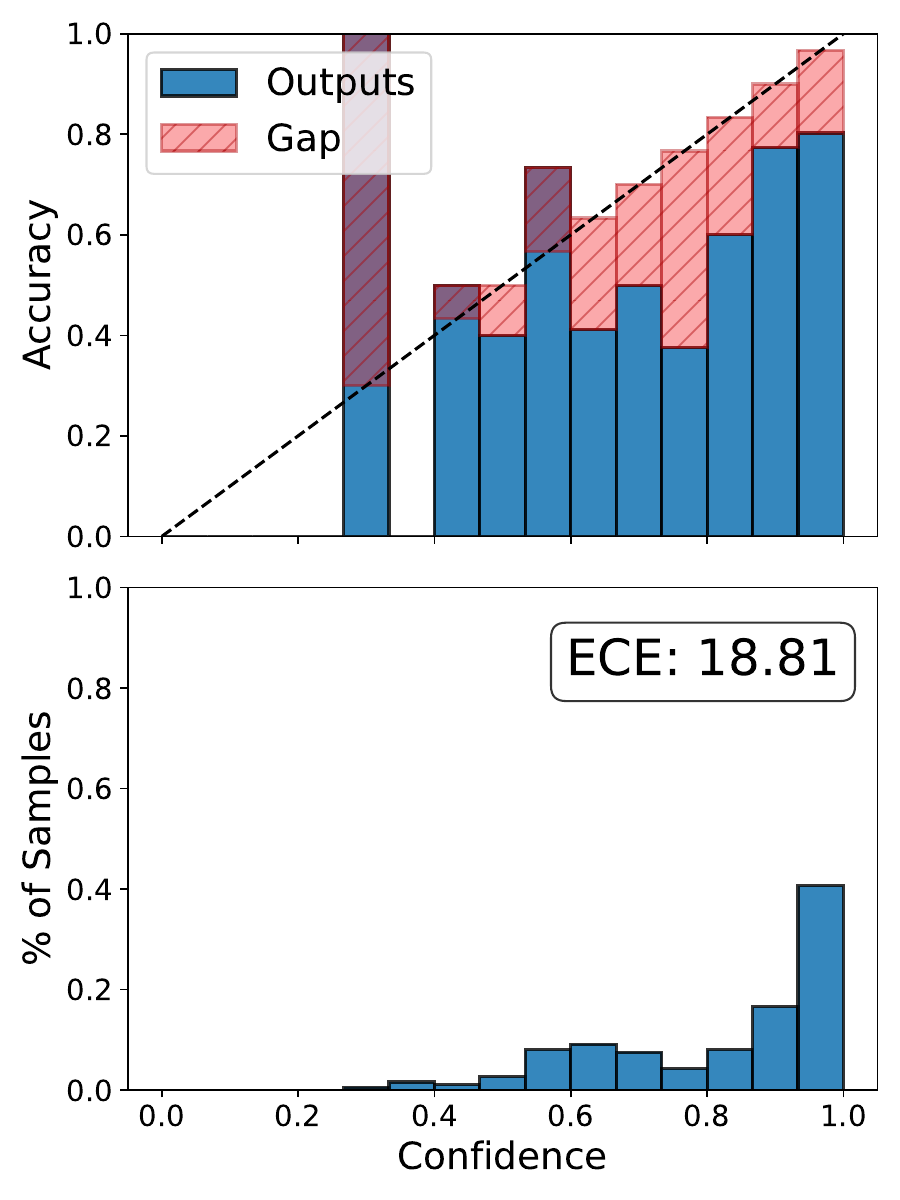}
         \caption{RCFT} \label{subfig:realibility-confidence-rcft}
     \end{subfigure}
     \hfill
    \begin{subfigure}[b]{0.24\textwidth}
         \centering
         \includegraphics[width=\textwidth]{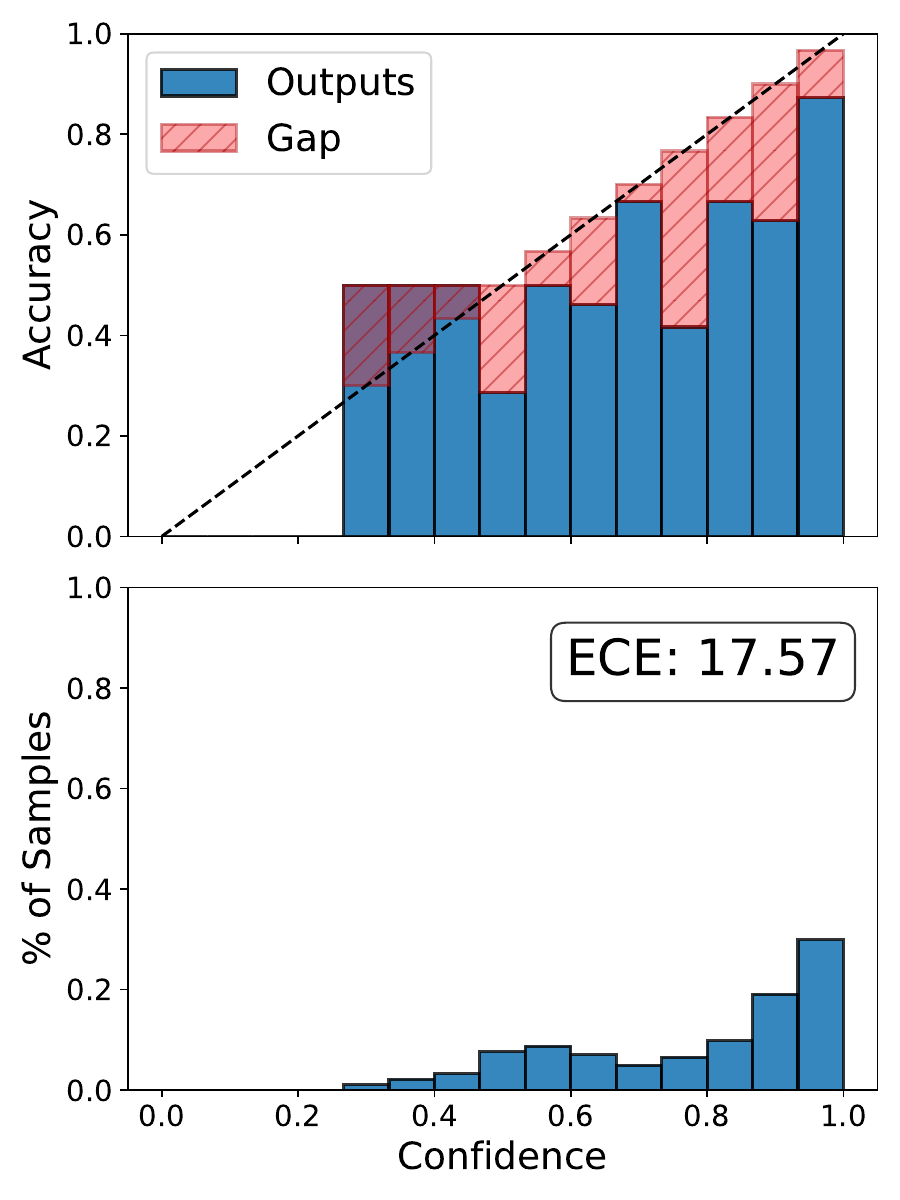}
         \caption{DPO + \projectnameshort} \label{subfig:realibility-confidence-cal}
     \end{subfigure}
    \caption{Reliability diagrams (top) compare predicted confidence with actual accuracy, while confidence histograms (bottom) show how often different confidence levels occur, using the same random seed. 
    The dashed line represents perfect calibration, deviations from it (marked with red bars annotated as "Gap") indicate miscalibration. 
    DPO produces strong overconfidence that remains even after RCFT post-hoc calibration, whereas \projectnameshort brings confidence closer to true accuracy, reducing calibration error (ECE).} 
    \label{fig:ece-plot-main}
\end{figure*}

Regularized Calibration-Aware Finetuning (RCFT)~\cite{xiao2025restoring} partially addresses this issue by incorporating confidence calibration signals after alignment as a supervised training phase. 
RCFT adopts an expectation–maximization procedure that estimates correctness from model confidence and retrains the model accordingly. 
While effective, this approach incurs a substantial computational overhead (18–39$\times$ the cost of DPO; Table~\ref{tbl:time-taken-main}) and still does not fully prevent confidence drift that occurs during preference optimization.

Moreover, confidence calibration is complicated by another important issue, i.e. a mismatch between calibration metrics and training objectives. 
Calibration metrics quantify the gap between reported confidence and empirical correctness, whereas LLMs are typically trained using token- or sequence-level likelihood objectives, which only influence confidence indirectly and do not explicitly constrain its alignment with empirical correctness. 
This mismatch means that optimizing likelihood or task accuracy does not guarantee calibrated confidence, highlighting the importance of accounting for miscalibration during training (Figure~\ref{subfig:realibility-confidence-dpo}). 

\subsection*{Contributions}

In this work, we propose a \projectname (\projectnameshort), a \textbf{calibration-aware preference optimization objective} that explicitly regularizes confidence during preference alignment.
A natural metric to target is the Expected Calibration Error (ECE), as it directly measures miscalibration. 
Since the original ECE objective is non-differentiable, we introduce \textbf{a differentiable, per-token objective} for correctness that serves as a continuous approximation of empirical accuracy.
\projectnameshort introduces a calibration-aware training signal that directly constrains predicted probabilities during alignment, \textbf{preventing confidence drift} during preference optimization and ensuring that high-confidence predictions remain meaningful.
We combine this surrogate loss with the preference optimization objective linearly, resulting in a end-to-end training pipeline that simultaneously performs preference alignment and confidence calibration. 
This approach ensures that improvements in training directly translate to more accurate confidence estimates, avoiding the limitations of post-hoc calibration.

Empirically, we evaluate \projectnameshort across multiple benchmarks covering a wide range of tasks, from selective prediction to answer reranking. We measure calibration using ECE as well as downstream utility in these tasks.
Results show that our method reduces miscalibration compared to strong baselines, achieving an improvement in ECE within $2.22\%$-$7.61\%$ when compared to DPO and $0.22\%$-$1.24\%$ when compared to the next best DPO baseline (between RCFT and DPO + BCE) for in-distribution experiments (Table~\ref{tbl:in-distribution-main}).
Across benchmarks, \projectnameshort improves or preserves task accuracy, achieving an average accuracy change of $+3.16\%$, with a maximum observed drop of $-1.33\%$ on tasks where all other DPO based methods also observe drops in accuracy.
Under distribution shift, \projectnameshort continues to reduce miscalibration, yielding ECE improvements of $10.44\%$–$1.46\%$ relative to DPO on out-of-distribution benchmarks while maintaining accuracy compared to DPO (Table~\ref{tbl:out-of-distribution-main}.
We also demonstrate a practical, test-time application of calibrated confidences via \emph{Confidence@k}, showing how token-level calibrated probabilities can be leveraged for Bayes-optimal selection among candidate outputs (Section~\ref{subsec:methodology-confidence-at-k},Table~\ref{tbl:conf-at-k-main}).

\section{Related Work}
\label{sec:related-work}
Confidence calibration~\cite{brier1950verification,murphy1972scalar} requires a model’s reported probability to reflect the empirical frequency of correctness. 
In practice, calibration is measured using metrics such as Expected Calibration Error (ECE) and its variants~\cite{guo2017calibration,kull2019beyond,nixon2019measuring,luo2022local}, which quantify the gap between predicted confidence and observed correctness. ECE is computed by binning predictions by confidence~\cite{naeini2015obtaining,guo2017calibration}:
\begin{equation}
\mathrm{ECE}(\theta) \triangleq
\sum_{m=1}^M \frac{|B_m|}{N}
\left|\mathrm{Z}(B_m) - \mathrm{c}_\theta(B_m)\right|,
\label{eq:ece}
\end{equation}
where $M$ is the number of confidence bins, $B_m$ is the set of predictions in bin $m$, $N$ is the total number of predictions, $\mathrm{c}_\theta(B_m)$ is the average predicted confidence in the bin, and $\mathrm{Z}(B_m)$ is the empirical accuracy of predictions in the bin. 
While widely used~\cite{guo2017calibration,huang2025calibrated}, ECE is non-differentiable due to binning and binary correctness, motivating our differentiable surrogates for training-time optimization.

Different training paradigms affect predicted probabilities differently, causing ECE behavior to vary between supervised and preference-based training. 
Motivated by this, we categorize prior work into: 
(a) supervised fine-tuning, which uses labeled data, and 
(b) preference optimization, which leverages human or learned signals. 
We review both, highlighting limitations and motivating our calibration-aware preference optimization method.

\subsection{Calibration during Supervised Fine-Tuning}

Post-hoc techniques adjust predictions after training without changing model parameters, including temperature scaling \cite{guo2017calibration}, isotonic regression \cite{zadrozny2002transforming}, and Dirichlet calibration \cite{kull2019beyond}.
Training-time approaches regularize confidence during supervised learning. Label smoothing \cite{szegedy2016rethinking}, focal loss \cite{lin2020focal}, and Mixup \cite{zhang2018mixup} reduce overconfidence by softening targets, reweighting hard examples, or smoothing decision boundaries but they provide limited control over the extent of calibration nor can they reduce miscalibration introduced during preference-alignment.
More recent differentiable surrogates, such as S-AvUC \cite{karandikar2021soft}, align predicted confidence with empirical accuracy but can incur bias due to the soft-binning artifacts, while other works target calibration in large-class classification \cite{coz2024confidence} but require training an auxiliary model to estimate the correctness target. 

However, these approaches operate only in labeled settings and do not persist under preference-based optimization, where logits can drift and absolute probabilities are unconstrained. 
In contrast, our method \projectnameshort embeds a differentiable calibration objective directly in preference optimization, preserving calibration throughout alignment rather than relying on post-hoc correction or supervised approximations.

\subsection{Calibration during Preference Optimization}

Preference optimization methods such as RLHF \cite{christiano2017deep} and DPO \cite{rafailov2023direct} optimize relative likelihoods between preferred and dispreferred outputs, but leave absolute probability scales unconstrained. 
This often drives logits to extreme values, inflating confidence and causing severe miscalibration.
\citeauthor{leng2025taming} show that preference optimization quickly degrades calibration even from well-calibrated or label-smoothed checkpoints, with similar confidence inflation reported in \cite{xiao2025algorithmic,xiao2024caldpo}. 
RCFT~\cite{xiao2025restoring} applies an EM-based post-hoc calibration, but relies on coarse bins, domain-specific assumptions, and incurs extra computational costs.

In contrast, our method regularizes absolute confidence during training via a continuous, per-token calibration objective. 
By integrating a differentiable calibration loss into DPO, we preserve preference ordering while anchoring probabilities to a correctness surrogate, preventing confidence drift without post-hoc correction, discretization, or domain-specific assumptions.

\section{Methodology}
\label{sec:methodology}

This section formalizes our approach to confidence calibration under preference optimization. 
(a) We first introduce \projectnameshort,
a differentiable per-token objective, that directly aligns predicted confidence with empirical correctness 
(\S~\ref{subsec:methodology-per-token-loss}). 
(b) We then incorporate this objective into the preference optimization framework, yielding a calibration-aware variant of Direct Preference Optimization that constrains absolute confidence while preserving preference ordering 
(\S~\ref{subsec:methodology-calibration-aware-preference-optimization}). 
(c) Finally, we present an inference-time mechanism that leverages calibrated token-level confidences to improve downstream decision-making (\S~\ref{subsec:methodology-confidence-at-k}). 

\paragraph{Notation.}
Let $\theta$ denote the parameters of an LLM, denoted as $\pi_\theta$.
For an input textual token sequence $x = (x_1, \dots, x_{t-1})$ with corresponding ground-truth output tokens $y^* = (y_t, \dots, y_T)$, the model induces a conditional distribution $\pi_\theta(\cdot \mid x_t)$ at each token position $t$. 
$(x,y^\star) \sim \mathcal{D}$ is drawn from a data distribution $\mathcal{D}$.
Let $\hat{y}_t = \arg\max_y \pi_\theta(y \mid x_t)$ denote the model's predicted token at position $t$, and let $\bar{y}_t$ denote the token with the highest predicted probability other than the ground-truth (i.e., $\bar{y}_t \neq y^*_t$). 
We denote the model’s predicted confidence as $c_\theta(x_t)~\triangleq~\pi_\theta(\hat{y}_t~\mid~x_t)$.
Let $Z_t \in \{0,1\}$ be a binary correctness indicator, where $Z_t = 1$ if $\hat{y}_t = y^*_t$ and $Z_t = 0$ otherwise. 
We further use $z(x)=\Pr(Z=1\mid x)$ to denote the true (but generally unknown) probability that the model's prediction is correct.
The goal of confidence calibration is to align the model's predicted confidence $c_\theta(x_t)$ with this underlying correctness probability as much as possible.

\subsection{A Differentiable Per-Token Calibration Loss}
\label{subsec:methodology-per-token-loss}
To ensure a model's confidence reflects empirical correctness, we target ECE (Eq.~\ref{eq:ece}), which quantifies the global gap between predicted confidence and observed accuracy. It also corresponds to a lower-bound on class-specific calibration errors.
While ECE captures this overall notion of calibration, it is not directly optimizable, since it depends on (a) hard binning of confidence values, where continuous confidence scores are discretized into fixed intervals, and (b) a binary-valued correctness indicator, both of which make it non-differentiable and cannot be directly used as a training objective.

To enable end-to-end optimization, we
(a) first define \textbf{a population-level calibration risk} across the full data distribution;
(b) \textbf{decompose it per token} to guide individual predictions;
(c) then introduce \textbf{a differentiable correctness surrogate} for stable gradients; and
(d) \textbf{extend per-token calibration to sequences}, ensuring reliable sequence-level confidence.
Together, these steps enable gradient-based training that preserves preference alignment while directly controlling confidence reliability.

\paragraph{A population measure of miscalibration.}
We first define a population calibration risk, $\mathcal{R}_{\mathrm{Cal}}$, which measures miscalibration before discretization. 
A natural population measure of miscalibration is the expected absolute deviation between confidence and correctness:
\begin{align}
\label{eq:pop-l1-deriv}
\mathcal{R}_{\mathrm{Cal}}(\theta)
\;=\;
\mathbb{E}_{(x,y^*)\sim\mathcal{D}}\Bigg[\frac{1}{T}\sum_{t=1}^T\,|\,c_\theta(x_t) - Z_t\,|\,\Bigg].
\end{align}
We refer to this quantity as the population $L_1$ calibration risk. 
This population risk measures the expected absolute deviation between predicted confidence and correctness across all tokens and sequences.
Unlike ECE, it is defined without binning and admits unbiased stochastic estimation.

Minimizing the population $L_1$ calibration risk directly reduces an upper bound on ECE
ensuring that calibration improvements translate directly to the ECE metric.
Formally, the population $L_1$ risk upper-bounds ECE by
$\text{ECE}(\theta)~\leq~\mathcal{R}_{\mathrm{Cal}}(\theta)$,
which follows from Jensen's inequality since $|\cdot|$ is convex:
\begin{align}
&\text{ECE}(\theta)
= \mathbb{E}_{x_t \sim \mathcal{D}}\Big[\,\big|\mathbb{E}_{Z\sim \Pr(Z\mid x_t)}[Z] - c_\theta(x_t)\big|\,\Big] \nonumber \\
&\leq \mathbb{E}_{x_t \sim \mathcal{D}}\Big[\mathbb{E}_{Z \sim \Pr(Z\mid x_t)}\big[|Z - c_\theta(x_t)|\big]\Big] = \mathcal{R}_{\mathrm{Cal}}(\theta).
\label{eq:ece-l1-upperbound}
\end{align}
The resulting upper bound establishes that the population $L_1$ calibration risk dominates ECE, motivating $\mathcal{R}_{\mathrm{Cal}}(\theta)$ as an optimization-friendly surrogate, i.e. it avoids binning, admits unbiased stochastic estimation, and preserves a direct relationship to the target calibration metric.

\textbf{Remark on ECE.}
We target ECE over other calibration estimation metrics because it serves as a standard global notion of probabilistic calibration. 
While there exist other refined metrics of calibration, such as classwise or thresholded/adaptive variants of ECE, which introduce additional conditioning or reweighting of the same underlying deviation $|\mathbb{E}[Z \mid C] - C|$, they measure the same underlying deviation between confidence and accuracy.
Rather than optimizing any specific variant, we focus on the population $L_1$ calibration risk $\mathcal{R}_{\mathrm{Cal}}(\theta)$ which provides an upper bound to the confidence-dependent reweightings and other subpopulations of ECE. 
We provide formal relationships between $\mathcal{R}_{\mathrm{Cal}}(\theta)$, ECE, and its variants are provided in Appendices~\ref{adx:appendix-properties-of-expected-calibration-error} and~\ref{adx:appendix-ece-and-l1-risk}.

\paragraph{Decomposing the per-token calibration risk.}
To obtain a training objective, we express the population $L_1$ calibration risk as a per-token loss, which provides a local calibration signal that aligns confidence with prediction correctness. 

The conditional $L_1$ risk decomposes according to the law of total expectation (full derivation in Appendix~\ref{subsec:decomposition-of-L1-risk}),
\begin{align}
&\mathbb{E}_{Z \sim \Pr(Z\mid x_t)}\Bigg[\frac{1}{T}\sum_{t=1}^{T}|c_\theta(x_t)-Z_t|\Bigg] \nonumber \\
&= \frac{1}{T}\sum_{t=1}^{T} \Big[z(x_t)(1-c_\theta(x_t)) + (1-z(x_t))c_\theta(x_t)\Big],
\label{eq:l1-decompose}
\end{align}
where $z(x_t) \coloneqq \Pr(Z=1 \mid x_t)$ denotes the true correctness probability and we have used $|c_\theta(x_t)-1|~=~1-c_\theta(x_t)$ since $c_\theta(x_t) \in [0,1]$. 
This decomposition shows that optimal calibration would match prediction confidence with $z(x_t)$ for each data point.

\paragraph{A surrogate for correctness.}
Decomposing the calibration risk requires a per-token objective $z(x_t)=\Pr(Z=1\mid x_t)$, but the true correctness probability $z(x_t)$ is not available during training.
A naive approach would be to replace $z(x_t)$ with the binary indicator $Z$. 
However, using $Z$ directly is not feasible, since the binary correctness indicator is non-differentiable and cannot provide a gradient signal for optimization.

To obtain a differentiable target, we introduce a continuous correctness surrogate derived from the model’s probability margin (see Appendix~\ref{adx:appendix-surrogate}):
\begin{equation}\label{eq:tildeZ-sigmoid}
\widetilde{z}(x_t) \;=\; \sigma\!\left({p_{y^\star}(x_t) - p_{\bar y}(x_t)}\right),
\end{equation}
where $p_{y^\star}(x_t)=\pi_\theta(y_t^\star\mid x_t)$ is the predicted probability of the ground truth token, $p_{\bar y_t}(x_t)=\max_{y_t\neq y_t^\star}\pi_\theta(y_t\mid x_t)$ is the highest predicted probability among tokens not equal to the ground truth, and $\sigma(u)=(1+e^{-u})^{-1}$ is the sigmoid function. 
This surrogate replaces the binary-valued outcome with a smooth, bounded signal (see Appendix~\ref{adx:appendix-gradient-stability}) that enables gradient-based optimization, while preserving the goal of confidence calibration. 

With this surrogate in place, we replace the unobservable per-token target $z(x_t)$ with $\widetilde{z}(x_t)$ in the per-token calibration loss, yielding a differentiable training objective. 
Since the population risk (Eq.~\ref{eq:pop-l1-deriv} and Eq.~\ref{eq:l1-decompose}) is recovered as the expected empirical loss over the dataset ($\mathcal{R}_{\mathrm{Cal}} = \mathbb{E}_{(x, y^\star)\sim\mathcal{D}}\big[ \mathcal{L}_\mathrm{Cal}(\theta; x, y^\star) \big]$), the sequence-level calibration loss becomes
{
\footnotesize
\begin{equation}
\mathcal{L}_{\mathrm{Cal}}(\theta ; x,y^*) = \frac{1}{T}\sum_{t=1}^{T}\Big[\widetilde{z}(x_t)(1-c_\theta(x_t)) + (1-\widetilde{z}(x_t))c_\theta(x_t)\Big]. \label{eq:per-token-cal-loss}
\end{equation}}

We further show in Appendix~\ref{subsec:properties-of-per-token-calibration-loss} that the per-token $L_1$ calibration loss is robust to surrogate noise (Proposition~\ref{prop:l1_stability}), yields valid gradients as the correctness surrogate is directionally consistent with the correctness indicator (Proposition~\ref{prop:directional_consistency}), and supports stable optimization under standard subgradient methods.

A correctness surrogate $\tilde u(x)$ is \emph{directionally consistent} with the true likelihood $u(x)$ if

\vspace{-15px} 

\[
\tilde u(x) > 0.5 \;\iff\; p_{y^\star}(x) > p_{\bar y}(x), 
\]

\vspace{-15px}

\hspace{0.2\textwidth}and

\vspace{-20px}

\[
\tilde u(x) < 0.5 \;\iff\; p_{y^\star}(x) < p_{\bar y}(x).
\]

\vspace{-10px}

That is, the surrogate correctly signals whether the model favors the correct token over its strongest competitor.

\paragraph{From token-level to sequence-level calibration.}
Our calibration objective aligns each token-level confidence $c_\theta(x_t)$ with the likelihood of token correctness and aggregates these losses across the sequence.
This mirrors standard autoregressive training, where multiplicative sequence-level quantities are optimized via additive surrogates.
When token-level confidences are calibrated, aggregating per-token confidences yields informative proxies for sequence-level correctness and reliability.
Thus, minimizing the per-token calibration loss $\mathcal{L}_{\mathrm{Cal}}(\theta;x,y^*)$ provides a natural and practical approach to confidence calibration for sequences.

\subsection{Calibration-Aware Preference Optimization}
\label{subsec:methodology-calibration-aware-preference-optimization}
We now show how to integrate the calibration objective derived in the previous section into preference alignment, specifically DPO, to mitigate the miscalibration that often arises during preference alignment.

DPO~\cite{rafailov2023direct} trains a model to prefer a higher-quality sequence $y^+$ over a lower-quality sequence $y^-$ given context $x$, using a fixed reference policy $\pi_\mathrm{ref}$ and scaling factor $\beta>0$.
The sequence-level preference score is
\[
r_\theta(x,y) \;:=\; \beta\big(\log\pi_\theta(y\mid x) - \log\pi_{\mathrm{ref}}(y\mid x)\big),
\]
and the corresponding pairwise DPO loss is
\[
\mathcal{L}_{\mathrm{DPO}}(\theta; x, y^+, y^-) = - \log \sigma(r_\theta(x,y^+) - r_\theta(x,y^-)).
\]
For sequences, $\log \pi_\theta(y \mid x)$ decomposes as a sum over token log-probabilities, preserving preference ordering even when candidate sequences differ in length.

\paragraph{Joint objective for confidence-calibrated DPO.}
While DPO enforces relative preferences between outputs, it does not put constraints on the values of token prediction probabilities, often leading to overconfident predictions.
To address this issue, we augment the DPO objective with our per-token calibration loss (Eq.~\ref{eq:per-token-cal-loss}), ensuring the model's confidence remains consistent with empirical correctness while preserving the preference information. 
The joint objective for ($x, y^{+}, y^{-}$) combines DPO with the per-token calibration term as
\begin{align}
\label{eq:dpo-with-calibration}
\mathcal{L}_{\mathrm{total}}&(\theta; x, y^+, y^-) = 
\mathcal{L}_{\mathrm{DPO}}(\theta; x, y^+, y^-) \nonumber \\
&+ \lambda \Big(\mathcal{L}_{\mathrm{Cal}}(\theta; x, y^+) + \mathcal{L}_{\mathrm{Cal}}(\theta; x, y^-)\Big),
\end{align}
where $\lambda>0$ balances preference alignment with confidence calibration, $\mathcal{L}_{\mathrm{Cal}}(\theta; x, y^+)$ uses confidence target     $\widetilde{z}(x_t)$ and $\mathcal{L}_{\mathrm{Cal}}(\theta; x, y^-)$ uses the confidence target $1-\widetilde{z}(x_t)$ to provide a supervised signal of incorrectness to the unpreferred response.

This joint loss, DPO + \projectnameshort, introduces a stabilizing effect on model probabilities. 
As shown formally in Appendix~\ref{subsec:preservation-of-preference-ordering-under-calibration-aware-dpo} (Proposition~\ref{prop:ordering-stability}), adding the per-token calibration term preserves the original DPO preference information under bounded calibration gradients, ensuring that calibration and preference alignment are compatible.
This result relies on the uniform bounds of the calibration gradient proven in Appendix~\ref{adx:appendix-bounded-log-prob-gradients} (Proposition~\ref{prop:cal-logprob-bound}). 

While $\mathcal{L}_{\mathrm{DPO}}$ only enforces relative preferences (pushing preferred sequences above dispreferred ones without controlling the magnitude of predicted probabilities), $\mathcal{L}_{\mathrm{Cal}}$ enforces absolute calibration, aligning the model's token probability predictions $c_\theta(x_t)$ with the true likelihood of correctness.
The per-token surrogate $\widetilde{z}(x_t)$ is derived from the probability margin between the ground-truth token and the most likely incorrect token, providing a scaled, differentiable target: high when the model is likely correct, low when it is likely wrong.
Applying this calibration loss to both preferred and dispreferred sequences ensures that the model’s confidence estimates are meaningful across the full output distribution, preventing overconfidence for unlikely predictions while maintaining alignment with empirical correctness.
Further in Table~\ref{tbl:time-taken-main}, we see that integrating \projectnameshort in DPO introduces no additional model parameters and incurs negligible overhead relative to RCFT.

\subsection{Confidence-Aware Inference via Confidence@k}
\label{subsec:methodology-confidence-at-k}

We now introduce Confidence@k, an inference-time mechanism we propose for using per-token calibrated confidence to improve the accuracy of selected predictions.

In many tasks~\cite{stiennon2020learning,ouyang2022training}, a model may generate multiple plausible outputs for the same input as candidates and select from them. 
For example, one can sample the top-$k$ outputs under the model's probability distribution, obtain the top-$k$ via beam search, or consider alternative outputs proposed by a base model or reference policy during online preference-based alignment, where multiple candidate outputs are evaluated for correctness. 
Selecting the most reliable among the candidates outputs may require more than just choosing the highest-likelihood sequence: in tasks such as multiple-choice question answering or mathematical reasoning, outputs that are superficially similar can differ substantially in correctness. 
By leveraging the calibrated token-level confidences produced by our per-token loss, \textit{Confidence@k systematically identifies the candidate most likely to be correct, improving downstream prediction accuracy}.

\paragraph{Confidence@k decision rule.}
Let $\mathcal{Y}$ denote the set of valid model outputs and $c_\theta(x,y) \in [0,1]$ denote the model's confidence that candidate output $y$ is correct for input $x$. 
Given a candidate set ${y_1,\dots,y_k}$, we define the \emph{Confidence@k} rule as
\begin{equation}
\hat{y}^{(k)}(x) \;\triangleq\; \arg\max_{y \in \{y_1,\dots,y_k\}} c_\theta(x,y).
\label{eq:confidence-at-k}
\end{equation}
This rule provides a simple, practical procedure to leverage token-level confidence calibration at inference: generate multiple candidates, compute confidence from per-token probabilities, and select the output with the highest confidence.
For our tasks, we consider label confidence to rerank outputs.
The theoretical justification for this selection rule, showing that it is Bayes-optimal under well-calibrated confidences, is detailed in Appendix~\ref{subsec:properties-bayes-optimal-selection with-confidence-at-k} (Proposition~\ref{prop:confk_optimal}).

\section{Experiments and Results}
\label{sec:experiments-and-results}
We evaluate \projectnameshort on 5 datasets, comparing it against state-of-the-art confidence calibration baselines. 
We detail the tasks, datasets and experimental setup in this section,
followed by the empirical results and discussion on findings. 

\begin{table*}[hbt!]
\caption{
Performance comparison for our method \projectnameshort against baselines for \textbf{in-distribution settings} on all datasets. 
$\uparrow$ means higher values are better and $\downarrow$ means lower values are better. 
Results are grouped by SFT-based and DPO-based methods.
We bold the highest accuracy and lowest ECE within the SFT/DPO baselines, respectively. 
}
\vspace{-0.2cm}
\centering
\setlength{\tabcolsep}{4pt} 
\begin{tabular}{lcccccccccc}
\toprule
& \multicolumn{2}{c}{\textbf{Reward Bench 2}} & \multicolumn{2}{c}{\textbf{SNLI}} & \multicolumn{2}{c}{\textbf{ANLI}} & \multicolumn{2}{c}{\textbf{TLDR}} & \multicolumn{2}{c}{\textbf{CommonsenseQA}} \\
\cmidrule(lr){2-3}\cmidrule(lr){4-5}\cmidrule(lr){6-7}\cmidrule(lr){8-9}\cmidrule(lr){10-11}
\textbf{Methods} & Acc. $\uparrow$ & ECE $\downarrow$ & Acc. $\uparrow$  & ECE $\downarrow$ & Acc. $\uparrow$  & ECE $\downarrow$ & Acc. $\uparrow$  & ECE $\downarrow$ & Acc. $\uparrow$  & ECE $\downarrow$ \\
 \midrule
SFT & 62.94 & 21.68  & 72.44 & 25.19 & 71.16 & 24.66 & 62.24 & 28.10 & \textbf{72.35} & 25.84 \\
+ Temperature Scaling & 60.16 & 20.49 & \textbf{74.48} & 24.50 & \textbf{72.36} & 23.98 & 61.26 & 27.80 & 70.50 & \textbf{24.25} \\
+ Label Smoothing & \textbf{63.85} & \textbf{16.19} & 73.24 & \textbf{21.44} & 71.88 & \textbf{15.76} & \textbf{62.81} & \textbf{27.70} & 70.90 & 42.20 \\
\midrule
+ DPO & \textbf{65.44} & 19.41 & 75.36 & 26.74 & 72.04 & 24.72 & 63.87 & 28.51 & 73.40 & 25.86 \\
+ RCFT & 64.58 & 15.33 & 72.14 & 22.94 & 72.12 & 23.74 & 63.72 & 27.63 & \textbf{73.45} & 24.83 \\
+ DPO + BCE & 60.20 & 15.78 & 76.92 & 20.26 & 72.24 & 24.19 & \textbf{64.68} & 26.66 & 73.39 & 23.85 \\
+ DPO + \projectnameshort (Ours) & 64.11 & \textbf{15.10} & \textbf{78.52} & \textbf{19.13} & \textbf{72.45} & \textbf{22.50} & 64.26 & \textbf{25.97} & 73.09 & \textbf{23.31} \\
\bottomrule
\end{tabular}
\label{tbl:in-distribution-main}
\end{table*}

\begin{table*}[hbt!]
\caption{Comparison of performance across our method (DPO + \projectnameshort) and the baselines (SFT, Temperature Scaling, Label Smoothing, DPO, DPO + BCE, RCFT) for \textbf{out-of-distribution datasets}, i.e. trained on MMLU and tested on other datasets. 
}
\vspace{-0.2cm}
\centering
\setlength{\tabcolsep}{4pt} 
\begin{tabular}{lcccccccccc}
\toprule
& \multicolumn{2}{c}{\textbf{Reward Bench 2}} & \multicolumn{2}{c}{\textbf{SNLI}} & \multicolumn{2}{c}{\textbf{ANLI}} & \multicolumn{2}{c}{\textbf{TLDR}} & \multicolumn{2}{c}{\textbf{CommonsenseQA}} \\
\cmidrule(lr){2-3}\cmidrule(lr){4-5}\cmidrule(lr){6-7}\cmidrule(lr){8-9}\cmidrule(lr){10-11}
\textbf{Methods} & Acc. $\uparrow$ & ECE $\downarrow$ & Acc. $\uparrow$  & ECE $\downarrow$ & Acc. $\uparrow$  & ECE $\downarrow$ & Acc. $\uparrow$  & ECE $\downarrow$ & Acc. $\uparrow$  & ECE $\downarrow$ \\
 \midrule
SFT & 10.69 & 39.83 & 36.25 & 18.48 & 30.80 & 17.38 & 24.70 & 33.51 & 37.39 & 19.87 \\
+ Temperature Scaling & \textbf{22.21} & \textbf{32.08} & \textbf{41.04} & 32.14 & \textbf{38.20} & 33.87 & \textbf{28.20} & 34.49 & \textbf{42.29} & 27.63 \\
+ Label Smoothing & 11.12 & 36.35 & 36.24 & \textbf{16.39} & 33.03 & \textbf{16.17} & 23.03 & \textbf{31.64} & 38.32 & \textbf{18.38} \\
\midrule 
+ DPO & 16.31 & 40.60 & \textbf{39.10} & 17.67 & \textbf{37.10} & 23.25 & 27.70 & 35.06 & \textbf{40.43} & 23.11 \\
+ RCFT & 16.57 & 34.98 & 36.39 & 30.61 & 36.24 & 22.94 & 27.40 & 34.82 & 39.65 & 20.93 \\
+ DPO + BCE & 16.04 & 33.73 & 31.20 & 27.23 & 35.90 & 21.08 & 28.70 & 32.46 & 39.28 & 19.33 \\
+ DPO + \projectnameshort (Ours) & \textbf{16.81} & \textbf{30.16} & 38.70 & \textbf{16.21} & 37.02 & \textbf{16.01} & \textbf{29.60} & \textbf{31.23} & 39.85 & \textbf{17.59} \\
\bottomrule
\end{tabular}
\label{tbl:out-of-distribution-main}
\end{table*}

\subsection{Experimental Setup}
\label{subsec:experimental-setup}
\textbf{Tasks, datasets and models.}
We experiment on text classification and multiple-choice question-answering tasks, following recent work on confidence calibration~\cite{xiao2025restoring,xiao2025algorithmic,chakraborty2024maxmin}. 
This setup, where one cast the problem as predicting the class/choice label via next token generation, explicitly links token generation probability with model confidence. 
This classification setting enables us to isolate confidence calibration behavior from effects of LLMs' intrinsic text generation mechanism, which introduces issues such as length bias and semantic ambiguity relevant in open-ended text generation tasks. 
We leave the adaptation of our method to such tasks to future work. 
The datasets we use are MMLU~\cite{hendrycks2021measuring}, ANLI~\cite{nie2019adversarial}, TLDR~\cite{trltldr2025preference}, Reward Bench 2 (RB2)~\cite{malik2025rewardbench2}, SNLI~\cite{bowman2015snli} and CommonsenseQA~\cite{talmor2019commonsenseqa}.

To construct both supervised and preference-based training signals, we generate chain-of-thought traces using the DeepSeek-R1 API \cite{deepseekai2025deepseekr1} 
\footnote{We make the synthetically generated data available \hyperlink{https://huggingface.co/datasets/NisargParikh/Synthetic_Dataset_For_MCQA}{here}}.
This formulation trains the model to produce intermediate reasoning before emitting a final answer while filtering spurious correct predictions and maintains a clear correspondence between the model’s predicted answer probabilities and its expressed confidence, while allowing the model to utilize its internal reasoning through explicit intermediate steps.
We specify the task formulation and dataset details in greater depth in Appendix \ref{adx:appendix-datasets}. 

\textbf{Models and evaluation metrics.}
We evaluate the performance of our methods and the baselines using Qwen 3 4B~\cite{yang2025qwen3} as the base model, enabling a fair comparison across methods. 
We mainly use two metrics, \textbf{ECE} and task \textbf{Accuracy}, which is defined at the token level as exact match with the reference label, for each multiple-choice question-answering dataset. 
Confidence@k is evaluated separately as a test-time procedure (see Experimental Settings).

\textbf{Baselines.}
We compare \projectnameshort against the following methods: (a) Standard Supervised Finetuning (\textbf{SFT}), (b) Supervised Finetuning with \textbf{Label Smoothing}~\cite{huang2025calibrated} with smoothing factor $0.1$, (c) \textbf{Temperature Scaling}~\cite{zhu2023calibration}, which is a no-training method applied on top of the SFT baseline, (e) \textbf{DPO}, using the model after SFT as the base model (same for all other preference optimization methods), (d) Regularized Calibration-Aware Fine-Tuning (\textbf{RCFT}), a supervised confidence calibration method which operates on top of the DPO baseline, (g) \textbf{DPO + BCE} where $BCE(\widetilde z (x), c_{\theta}(x))$ is the binary cross-entropy between model confidence and our correctness surrogate, and (h) \textbf{DPO + \projectnameshort} (Ours).

All methods share a common SFT initialization without Label Smoothing unless stated otherwise.
Temperature Scaling is applied post-hoc on the SFT outputs to improve calibration without further training. 
The DPO baseline uses the SFT model without Label Smoothing as its base model and RCFT is implemented as a supervised phase of training on top of the DPO trained model.
For calibration-aware methods, including DPO+\projectnameshort and DPO+BCE, the calibration objective is applied during preference optimization, ensuring that the model learns to align its predicted confidence with a correctness surrogate while retaining the original preference signals. Further details of hyperparameters and Baselines are provided in 
\ref{adx:appendix-experimental-details}. 

\textbf{Experimental settings.}
We perform experiments under two main data distribution regimes: (a) in-distribution, where models are trained and tested on the same dataset, and (b) 
out-of-distribution, where models are trained on MMLU and tested on other datasets.
We perform out-of-distribution experiments to test the generalization of calibration methods.
In addition, we evaluate \textbf{Confidence@k}, a test-time \emph{selection rule} rather than a training method, in which $k$ candidate outputs are generated per input and the output with the highest predicted confidence is selected.
This procedure does not involve any additional training and is applied uniformly on top of all trained models to assess the practical impact of confidence calibration at inference time.
For Confidence@k experiments, We use $k \in \{4, 8\}$ and perform temperature scaling to select the best temperature for each baseline. 
Greedy decoding is used for all other evaluations aside from Temperature Scaling in Tables~\ref{tbl:in-distribution-main} and \ref{tbl:out-of-distribution-main}. 
Experiments are repeated over five random splits, and mean values are reported; further details and standard deviations are provided in Appendix~\ref{adx:appendix-experimental-details}.

\begin{table*}[hbt!]
\caption{Comparison of performance for Confidence@k with temperature scaling mechanism across our method (DPO + \projectnameshort) and the baselines (SFT, Label Smoothing
, DPO, DPO + BCE, RCFT) for in-distribution datasets. We report Accuracy across all datasets. We consider label token probabilities to rank model outputs. We report accuracy for $k\in\{4,8\}$.}
\vspace{-0.2cm}
\centering
\setlength{\tabcolsep}{5pt} 
\begin{tabular}{lcccccccccc}
\toprule
 &  \multicolumn{2}{c}{\textbf{Reward Bench 2}}  & \multicolumn{2}{c}{\textbf{SNLI}} & \multicolumn{2}{c}{\textbf{ANLI}} & \multicolumn{2}{c}{\textbf{TLDR}} & \multicolumn{2}{c}{\textbf{CommonsenseQA}} \\
\cmidrule(lr){2-3}\cmidrule(lr){4-5}\cmidrule(lr){6-7}\cmidrule(lr){8-9}\cmidrule(lr){10-11}
\textbf{Methods} & 4  & 8  & 4  & 8  & 4   & 8  & 4   & 8  & 4   & 8  \\
 \midrule
SFT  & 62.91 & 65.31 & 71.49 & 72.45 & 70.25 & 72.38 & 60.80 & 63.90 & 66.20 & 71.38\\
+ Label Smoothing  & 65.78 & 66.91 & 70.42 & 75.20 & 71.34 & 78.45 & 62.20 & 65.80 & 55.64 & 59.40 \\
\midrule
+ DPO  & 64.71 & 65.43 & 68.32 & 68.49 & 64.41 & 68.51 & 60.63 & 62.40 & 63.40 & 67.47 \\
+ RCFT  & 66.10 & 68.74 & 69.39 & 72.67 & 65.80 & 76.52 & 64.24 & 67.52 & 67.75 & 71.82 \\
+ DPO + BCE  & 65.16 & 67.28 & 74.81 & 79.40 & 64.76 & 75.83 & 63.73 & 66.86 & 67.48 & 72.60 \\
+ DPO + \projectnameshort (Ours) & \textbf{69.53} & \textbf{72.19} & \textbf{77.32} & \textbf{82.80} & \textbf{73.51} & \textbf{81.45} & \textbf{65.37} & \textbf{70.34} & \textbf{73.67} & \textbf{74.33} \\
\bottomrule
\end{tabular}
\label{tbl:conf-at-k-main}
\end{table*}

\subsection{Results and Discussion}
\label{subsec:experimental-results}
Tables~\ref{tbl:in-distribution-main},~\ref{tbl:out-of-distribution-main} and \ref{tbl:conf-at-k-main} compare \projectnameshort against standard DPO and other baseline calibration techniques across multiple benchmarks. 
We analyze both in-distribution and out-of-distribution settings, focusing on the tradeoff between calibration quality, measured by ECE, and task accuracy. We summarize the main results below. 

\textbf{Standard calibration techniques have limitations.} 
Standard calibration techniques (temperature scaling, label smoothing, RCFT) show inconsistent improvements in calibration relative to their corresponding uncalibrated baselines (SFT or DPO). 
While DPO generally increases task accuracy (\(+0.9\%\) to \(+2.9\%\) across datasets), it degrades confidence calibration, increasing ECE for most benchmarks. 
This degradation can be detrimental to downstream applications. Table~\ref{tbl:conf-at-k-main} shows that standard DPO frequently degrades reliability in decision-making scenarios compared to the initial SFT model. 
RCFT and DPO+BCE reduce ECE by $0.5$-$6.5\%$ relative to DPO but sometimes at the cost of substantial accuracy drops, e.g., RCFT losing $3.2\%$ on SNLI and DPO+BCE losing $5.2\%$ on Reward Bench~2. 

\textbf{\projectnameshort improves calibration.} 
Across all in-distribution benchmarks (Table~\ref{tbl:in-distribution-main}), augmenting DPO with \projectnameshort consistently yields substantial reductions in ECE while preserving or slightly improving accuracy. Compared to vanilla DPO, \projectnameshort reduces ECE by $2.22$-$7.61\%$ across datasets.
We further observe that \projectnameshort does not introduce new accuracy degradation on top of the DPO objective.
On datasets where both RCFT and DPO+BCE reduce accuracy relative to DPO, such as Reward Bench~2 and CommonsenseQA, \projectnameshort incurs smaller drops ($-0.24\%$ and $-0.58\%$, respectively) than RCFT ($-0.86\%$ and $-0.78\%$) and DPO+BCE ($-5.24\%$ and $-1.15\%$).
In contrast, on SNLI and ANLI, \projectnameshort improves accuracy over DPO by $+3.16\%$ and $+0.41\%$, respectively, despite achieving better confidence calibration.

Compared to RCFT, which also operates on top of DPO, \projectnameshort consistently reduces ECE, with reductions of $0.33$-$3.81\%$ across all datasets. 
Notably, RCFT incurs an accuracy drop on SNLI ($-3.22\%$) and smaller drops on Reward Bench~2 and TLDR, whereas \projectnameshort preserves or slightly improves accuracy on these benchmarks.
These results suggest that optimizing calibration within preference optimization provides a better calibration-accuracy tradeoff than post-alignment supervision.

\textbf{\projectnameshort shows out-of-distribution robustness.}
The benefits of \projectnameshort extend to the out-of-distribution setting (Table~\ref{tbl:out-of-distribution-main}) as well, where models are trained on MMLU and evaluated on held-out benchmarks. In this more challenging regime, \projectnameshort achieves the lowest ECE on every dataset, with reductions of $1.46$-$10.44\%$ relative to DPO. 
At the same time, accuracy remains comparable to or exceeds that of DPO; \projectnameshort changes accuracy by $-0.28\%$ to $+1.9\%$ across benchmarks, whereas RCFT exhibits larger (and more negative) deviations from DPO, ranging from $-2.71\%$ to $+0.26\%$ on the same datasets.

We also observe distinct failure modes in other regularization methods, while DPO+BCE and RCFT offer reasonable in-distribution calibration, their OOD performance degrades significantly on SNLI and Reward Bench~2, resulting in ECE values which are $11.02$-$14.4\%$ and $3.57$-$4.82\%$ worse than \projectnameshort, which remains consistently calibrated across OOD settings.
These results suggest that the confidence calibration performance of \projectnameshort can generalize robustly even in the presence of distribution shifts.

\textbf{\projectnameshort helps the Confidence@k decision rule with effective calibration.} 
Table~\ref{tbl:conf-at-k-main} shows that \projectnameshort consistently improves top-$k$ task accuracy over existing baselines. 
Across all datasets, it outperforms RCFT and DPO+BCE by absolute gains of $2.82\%$-$7.71\%$ across datasets.
These trends hold for $k=4$ and $k=8$, demonstrating that calibration improvements induced by \projectnameshort translate directly into more reliable Confidence@$k$ predictions, while preserving or enhancing Top-$1$ accuracy relative to other methods.
Crucially, \projectnameshort reverses the degradation caused by standard preference alignment; on ANLI and SNLI ($k=8$), where DPO underperforms the unaligned SFT baseline by $3.87\%$ and $3.96\%$, \projectnameshort surpasses SFT's performance, achieving absolute gains of $+9.07\%$ and $+10.35\%$.


\section{Conclusions and Future Work}
\label{sec:conclusion}
In this work, we proposed \projectnameshort, a calibration-aware preference optimization framework that integrates confidence calibration objectives directly into the alignment process. It is necessary since standard preference optimization methods such as Direct Preference Optimization can degrade confidence calibration.
Across multiple benchmarks, \projectnameshort consistently reduces expected calibration error while maintaining or improving task accuracy, outperforming existing calibration methods and is under distribution shift. 
We also show that \projectnameshort can benefit LLMs in downstream test-time scaling methods. 
We hope this work encourages future research on jointly optimizing preference alignment and model uncertainty in LLMs.

There are plenty of avenues for future work. 
First, \projectnameshort calibrates model confidence without assumptions on sequence structure. Therefore, extending \projectnameshort to open-ended text generation and multimodal settings remains an important challenge, where sequence-level confidence can interact with length bias and semantic uncertainty.
Secondly, we focused only on static datasets rather than online or iterative alignment settings, such as RLHF. Therefore, we can investigate whether \projectnameshort helps confidence calibration throughout iterative training as well.
Finally, we can explore how confidence calibration can benefit more downstream applications beyond test-time-scaling.

\clearpage
\section*{Impact Statement}
\label{sec:impact-statement}
Our work improves the calibration of large language models while preserving alignment with human preferences, enhancing the reliability of their confidence estimates. 
Better-calibrated confidence can support safer deployment in real-world and high-stakes applications by reducing overconfident errors. 
However, improved calibration does not guarantee correctness, and misuse could lead to models that appear trustworthy without being more accurate. 
We encourage careful consideration of these limitations when applying our method.
While our primary focus is on improving confidence calibration, the differentiable token-level correctness surrogate could also be used in other contexts where estimating or guiding token-level accuracy is useful, such as data selection, active learning, or fine-grained reward shaping.

\section*{Limitations}
\label{sec:limitations}
Our approach relies on a differentiable surrogate for correctness derived from model probabilities; while we show robustness to bounded surrogate noise, extreme misestimation of probabilities may weaken the calibration signal. 
Our synthetic data generation makes the assumption that all reasoning traces which provides the correct label token are valid, quality of COTs can affect extent of calibration.
Correctness is statically verifiable in the MCQ task regime, this assumption may not hold and empirical correctness estimation could be noisy or unreliable in the free-from text generation, adding further noise to the objective.
In some tasks, empirical sequence-level correctness may not decompose into local token-level behaviors, so token-level confidence estimates may be insufficient to fully capture sequence-level reliability.


\bibliography{references}
\bibliographystyle{icml2026}

\newpage
\appendix
\onecolumn
\section{Theoretical Analysis}
\label{sec:theoretical-analysis}

This section provides formal analysis of our framework's key theoretical properties: 
(a)~the robustness and directional consistency of the per-token calibration loss (\S~\ref{subsec:methodology-per-token-loss}), 
(b)~the preservation of relative sequence preferences under calibration-aware DPO (\S~\ref{subsec:methodology-calibration-aware-preference-optimization}), and 
(c)~the Bayes-optimality of candidate selection under Confidence@k (Section~\ref{subsec:methodology-confidence-at-k}).

\subsection{Properties of the Per-Token Calibration Loss}
\label{subsec:properties-of-per-token-calibration-loss}
Here we formalize key properties of the per-token calibration loss, showing that it is robust to surrogate noise, consistently directs confidence updates, and produces stable, convergent gradients during optimization.

\paragraph{Property 1: The $L_1$ calibration loss is robust to surrogate noise.}
Our goal is to align model confidence with empirical correctness. 
In practice, the differentiable surrogate $\widetilde{z}(x)$ is an imperfect estimate of the true correctness $Z$, subject to bounded approximation error or ``contamination'' \cite{ghosh2017robust}. 
Formally, we model this as:
\[
\widetilde{z}(x) = z(x) + \epsilon(x), \qquad |\epsilon(x)| \le \delta < 1,
\]
with $\widetilde{z}(x)$ sign-consistent with correctness whenever the margin is nonzero. 

When the target is noisy, squared ($L_2$) losses amplify large deviations, producing gradients proportional to the surrogate error $c_\theta(x) - \widetilde{z}(x)$. 
This can destabilize confidence updates, especially in the presence of high-variance or contaminated surrogates. 
In contrast, the $L_1$ loss generates bounded gradients, limiting the influence of extreme surrogate errors and ensuring stable, reliable updates:
\begin{proposition}[Gradient Stability]
\label{prop:l1_stability}
The subgradients of the per-token $L_1$ calibration loss with respect to $c_\theta(x)$ satisfy
\begin{align}
    \Bigg|\frac{\partial \mathcal{L}_{\mathrm{Cal}}}{\partial c_\theta(x)}\Bigg| = \big| 1 - 2\widetilde{z}(x) \big| \leq 1
\end{align}
In contrast, $L_2$ gradients scale as $2(c_\theta(x) - \widetilde{z}(x))$. This boundedness ensures that large surrogate errors $\epsilon(x)$ do not destabilize confidence updates.
\end{proposition}
A detailed derivation is provided in Appendix~\ref{adx:appendix-gradient-stability}.
Furthermore, we show in Appendix~\ref{adx:appendix-surrogate-contamination} (Proposition~\ref{prop:symmetry_contam}) that the loss satisfies a symmetry condition that further enhances noise tolerance.
 
\textbf{Takeaway}: 
Using an $L_1$ calibration loss provides stable confidence updates even when the surrogate $\tilde{z}(x)$ is noisy. 
As proven in Appendix~\ref{adx:appendix-l1-and-l2} (Theorem~\ref{thm:robust-contamination}), the $L_1$ objective is statistically robust to outlier contamination in the surrogate, ensuring the model aligns with the underlying correctness probability without being destabilized by hallucinations.

\paragraph{Property 2: The probability margin surrogate preserves directional consistency.}
To ensure that confidence updates move in the correct direction, it is crucial that the surrogate $\widetilde{z}(x)$ consistently indicates whether the model’s prediction is likely correct or incorrect. 
We construct $\widetilde{z}(x)$ from the probability margin $p_{y^\star}(x)-p_{\bar y}(x)$, which is bounded in $[-1,1]$ and captures how much more probability mass the model assigns to the correct token compared to its strongest competitor.

\begin{proposition}[Directional Consistency of the Margin Surrogate]
\label{prop:directional_consistency}
Let
\[
\widetilde{z}(x) = \sigma\!\left({p_{y^\star}(x) - p_{\bar y}(x)}\right).
\]
Then $\widetilde{z}(x) > 1/2$ if and only if $p_{y^\star}(x) > p_{\bar y}(x)$, ensuring that the gradient of the calibration loss adjusts confidence in the correct direction.
\end{proposition}

\textbf{Takeaway}: Even under the bounded surrogate noise described by the contamination model, $\widetilde{z}(x)$ always aligns with the direction of the true correctness indicator, guaranteeing that confidence updates consistently move toward increasing correctness.

\paragraph{Property 3: 
The per-token calibration loss generates well-scaled, non-reinforcing gradients.
}
Although the per-token calibration loss is piecewise smooth, it can be non-differentiable when multiple tokens tie for the maximum probability or when the probability margin is exactly zero. 
These events occur with probability zero (Lebesgue-measure zero) and do not affect practical optimization~\cite{rockafellar1997variational}. 
Because the loss satisfies the standard assumptions of nonconvex, nonsmooth analysis, subgradient-based methods with deterministic tie-breaking, such as AdamW, are guaranteed to converge to a stationary point \citep{bertsekas1997nonlinear,davis2020stochastic,bolte2020conservative}.

\textbf{Convergence under Token Switching.} 
A potential concern with the margin-based surrogate is that the identity of the highest-probability incorrect token $\bar{y}_t$ may switch during training, introducing non-smoothness. 
In Appendix \ref{adx:appendix-convergence-properties}, we formally show that despite these switching events, the loss function satisfies the Kurdyka-\L{}ojasiewicz (KL) property. 
Consequently, standard stochastic subgradient methods are guaranteed to converge to a stationary point, ensuring that the discrete nature of token selection does not destabilize training.

\textbf{Takeaway}: Despite isolated non-differentiable points, the calibration loss produces stable, convergent updates under standard subgradient optimization.

\subsection{Preservation of Preference Ordering under Calibration-Aware DPO}
\label{subsec:preservation-of-preference-ordering-under-calibration-aware-dpo}
When training with DPO, the goal is to enforce relative preferences between sequences: 
the model should assign higher scores to preferred outputs $y^+$ than to dispreferred ones $y^-$. 
However, augmenting the DPO objective with the per-sequence calibration loss $\mathcal{L}_{\mathrm{Cal}}$ (Equation~\ref{eq:per-token-cal-loss}) changes absolute confidence values. 
If these calibration-induced updates are too large, they could in principle reverse the intended preference ordering. 

Our goal here is to formally show that, under bounded calibration gradients, adding $\mathcal{L}_{\mathrm{Cal}}$ preserves the original preference ordering.
To formalize this, we define the DPO margin for a preference pair $(y^+,y^-)$:
\[
\Delta_{\mathrm{DPO}}(x)
\;:=\;
r_\theta(x,y^+) - r_\theta(x,y^-),
\]
which measures how strongly the model prefers $y^+$ over $y^-$. We assume that this margin is bounded below by a positive constant $\Delta_{\min} > 0$, ensuring that the preference is non-ambiguous for all training pairs.

Next, we bound the influence of the calibration loss on the model’s sequence log-probabilities. 
If the per-sequence calibration gradient satisfies
\[
\bigl|\partial L_{\mathrm{Cal}} / \partial \log \pi_\theta(y\mid x)\bigr|
\;\le\; \frac{|y|}{4}
\quad \text{for all } (x,y),
\]
then the perturbation introduced by $\lambda \mathcal{L}_{\mathrm{Cal}}$ cannot exceed a threshold that would flip the preference ordering. 

Combining these conditions, we obtain the following guarantee:

\begin{proposition}[Preference Ordering Stability]
\label{prop:ordering-stability}
Under the assumptions above, the augmented objective, the augmented objective
\(L_{\mathrm{DPO}} + \lambda L_{\mathrm{Cal}}\)
preserves the DPO preference ordering for all training pairs whenever
\[
\lambda < \frac{2\Delta_{\min}}{|y|}.
\]
\end{proposition}

\textbf{Takeaway}: 
By carefully controlling the calibration weight $\lambda$ and using bounded per-sequence gradients, \projectnameshort can adjust absolute confidence to improve calibration without disrupting the relative preferences learned by DPO. 
This ensures that preference alignment and confidence calibration are compatible objectives rather than competing ones.

\subsection{Bayes-Optimal Selection with Confidence@k.}
\label{subsec:properties-bayes-optimal-selection with-confidence-at-k}
At inference time, a model may generate multiple plausible outputs for a single input. 
Simply selecting the sequence with the highest likelihood does not always maximize the chance of correctness, because sequences with similar likelihoods can differ substantially in actual accuracy. 
To address this, we use the calibrated per-token confidences to compute a sequence-level confidence score, $c_\theta(x, y)$, which estimates the probability that a candidate sequence $y$ is correct given input $x$.

The Confidence@k selection rule is shown in Section~\ref{subsec:methodology-confidence-at-k}, which chooses the candidate with the highest estimated probability of correctness. 

\begin{proposition}[Bayes-Optimality of Confidence@k]
\label{prop:confk_optimal}
If the model’s confidence estimates satisfy
\[
c_\theta(x, y) \;=\; \Pr(y\ \text{is correct} \mid x),
\]
then selecting
\[
\widehat{y}^{(k)}(x) = \arg\,max_{y \in \{y_1,\ldots,y_k\}} 
c_\theta(x, y)
\]
minimizes the probability of choosing an incorrect candidate from the set.
\end{proposition}
\textbf{Takeaway}: Confidence@k provides a principled, practical inference procedure that leverages calibrated token-level predictions to reliably select the most likely correct output, ensuring that improvements in calibration directly translate to better downstream decision-making.

\newpage

\section{Datasets and Task formulation}
\label{adx:appendix-datasets}
\subsection{Datasets}

In this section we specify the datasets we use, statistics about them, how we use them (splits, formatting, etc.) and, the prompts used when training or testing on the specified datasets. 

\textbf{Reward Bench 2} \cite{malik2025rewardbench2} is a benchmark made to evaluate Reward models across 6 domains (Factuality, Precise Instruction Following, Math, Safety, Focus, Ties).
Each sample in the benchmark consists of a user prompt paired with one preferred completion and multiple unpreferred completions.
We reformulate each instance into a multiple-choice question task, where the model selects the preferred completion among the provided options.
We consider $\{A,B,C,D\}$ as the annotations for each candidate completions, where we include one preferred and three unpreferred completions.
The preferred completions is assigned randomly to one of the annotations and the unpreferred completions are assigned to the rest of the annotations.
We mix samples from all the domains to get a varied set of questions. 
Reward Bench 2 only has 1.87K samples, which we shuffle with a random seed and divide in 3 splits (Train 80\%, Validation 10\%, Test 10\%). 

\textbf{SNLI} \cite{bowman2015snli} is a benchmark for natural language inference (NLI). 
The dataset consists of human-annotated premise–hypothesis pairs collected via crowdsourcing, with labels assigned under explicit semantic entailment guidelines.
Each sample contains: (a) a premise, (b) a hypothesis and, (c) a label (0 for entailment, 1 for neutral and 2 for contradiction). 
We consider $\{entailment, neutral, contradiction\}$ as the candidate annotations for each sample. 
SNLI has 3 natural splits: (a) Train (550K samples), (b) Validation (10K samples) and (c) Test (10K samples). 

\textbf{ANLI} \cite{nie2019adversarial} is a NLI benchmark which is generated via three rounds of an iterative, adversarial human-and-model-in-the-loop procedure.
ANLI was constructed to address annotation artifacts in earlier NLI benchmarks by adversarially targeting model weaknesses.
ANLI shares the same three-way label space as SNLI (entailment, neutral, contradiction).
We consider $\{entailment, neutral, contradiction\}$ as the candidate annotations for each sample. 
We focus on Round 1 to control for dataset size and distributional shift while retaining the adversarial nature of the benchmark.
The first round has 3 splits: (a) Train (16.9K samples), (a) Dev (1K samples), (a) Test (1K samples). 
We use the Dev split as the Validation split. 

\textbf{TLDR} \cite{trltldr2025preference} is a human preference dataset derived from Reddit posts and associated summaries.
The dataset consists of Reddit posts paired with candidate summaries, annotated with human preference judgments.
Each instance contains a prompt (the Reddit post) and two candidate summaries, a chosen summary and a rejected summary.
We consider $\{A,B\}$ as the annotations for each candidate completions, where we include one chosen and one rejected summary.
We randomly assign the chosen summary to one of the annotations and the rejected summary to the other.
We use TLDR as a real-world preference dataset to evaluate confidence calibration in settings where supervision is inherently subjective.
We use the public \texttt{trl-lib/tldr-preference} dataset hosted on Hugging Face.
The TLDR dataset has 2 splits: (a) Train (92.9K samples) and, (b) Validation (86.1K samples). 
We shuffle and split the Validation set with a given seed in half to create the Validation and Test sets. 

\textbf{Commonsense QA} \cite{talmor2019commonsenseqa} is a multiple-choice question answering dataset that tests a diverse set of commonsense knowledge. 
Commonsense QA has 12,102 questions with one correct and 4 distractor answers. 
We consider $\{A,B,C,D,E\}$ as the annotations for each candidate answer, where we include one correct answer and four distractor answers.
We randomly assign the correct answer to one of the annotations and the distractor answers to the rest of the annotations.
The dataset has 3 splits: (a) Train (9.74K samples), (b) Validation (1.22K samples) and, (c) Test (1.14K samples) splits.
The Test split does not have answer labels, hence we shuffle and split the Validation set with a given seed in half to form the Test and Validation split for our experiments.

\textbf{MMLU} \cite{nixon2019measuring} is a multi-task test consisting of multiple-choice questions from 57 separate domains. 
Some domains include the humanities, social sciences, mathematics, college physics, law , high school physics, etc. 
MMLU evaluates performance across a wide range of academic subjects, requiring broad knowledge and cross-domain generalization.
Each sample in the dataset includes a question, 4 options and the correct answer label.
Performance and confidence can vary substantially across subjects, making MMLU suitable for analyzing calibration under domain heterogeneity.
We use the entire corpus as it provides a more heterogeneous distribution of knowledge. 
The dataset is separated by subjects, but we use the entire corpus of questions which is split in 4 splits: (a) Auxilliary Train (99.8K samples), (b) Dev (285 samples), (c) Validation (1.53 samples) and, (d) Test (14K samples) splits.
We use the Auxilliary Train as the Train split. \\

\textbf{Prompt Template}

We use a generic prompt template across all datasets, shown below, with dataset-specific values for the expert role, task description, and candidate choices.
The placeholders in the template (e.g., expert role, number of choices, and annotation labels) are populated according to the requirements of each benchmark, while the overall structure and response constraints remain fixed.
Although the template includes a reasoning section for completeness, models are evaluated solely based on the final selected answer, and intermediate reasoning is not used for scoring, i.e. the calculation of ECE and Accuracy only considers the final Label Token.
We use this format of COT + final Label Token to evaluate the model's language modeling abilities while getting an exact estimation of confidence based metrics. 
This abstraction allows us to describe the evaluation protocol without revealing dataset-specific prompts or proprietary annotations.

\begin{tcolorbox}[colback=gray!5!white,colframe=black!75!black,title= Prompt]
You are an $\{$Expert Role$\}$. 

Given a prompt and $\{$N$\}$ choices ($\{$Candidate Choices$\}$), select the most appropriate answer.

\vspace{5px}

First, consider the suitability of each candidate choice with respect to the given prompt.

Then, output the final selected answer.

\vspace{5px}

Please respond in the following format:

\vspace{5px}

$<$think$>$
[Your reasoning about which choice is most appropriate]
$<$/think$>$

\vspace{5px}

$<$answer$>$
[$\{$Candidate Choices$\}$]
$<$/answer$>$

\vspace{10px}

$\{$\textbf{Prompt Annotation}$\}$: $\{$\emph{prompt query}$\}$

\vspace{5px}

$\{$\textbf{Annotation A}$\}$: $\{$option A$\}$

$\{$\textbf{Annotation B}$\}$: $\{$option B$\}$

\hspace{5mm}\vdots
\end{tcolorbox}

\textbf{Synthetic Supervised Data}

We collected supervised training data by querying DeepSeek-R1 3.2 API across the above mentioned benchmark datasets(Reward Bench 2, SNLI, ANLI, MMLU, CommonSenseQA, and TL;DR) and filtered responses to retain only those that match ground truth labels to gather  reasoning traces aligned with correct decisions. 
For each dataset, we construct task-specific system prompts that are equivalent to the above per-dataset prompts which instruct the model to produce answers in a structured format (e.g., `$<$answer$>$A$<$/answer$>$` for multiple-choice questions). 
Each supervised example consists of: (i) the original input prompt, (ii) the task-specific system prompt, (iii) the full model-generated response, (iv) the parsed categorical answer, (v) the ground truth label, and (vi) the extracted reasoning trace.

\textbf{Synthetic Preference Data}

We generate preference datasets for offline preference optimization (e.g. DPO) training by sampling multiple diverse responses from DeepSeek-R1 3.2 API for each example and constructing chosen/rejected pairs using correctness with respect to the ground truth as a weak preference signal. 
For each example, we generate multiple samples (typically 10) using varying sampling temperatures (uniformly sampled from a fixed range), then parse all responses and categorize them into "chosen" (correct) and "rejected" (incorrect) categories based on comparison with ground truth labels.
Examples are retained only if they meet minimum thresholds for both correct and incorrect responses (typically at least 1 correct and 3 incorrect responses), ensuring sufficient diversity for effective preference learning. 
This filtering ensures that each training instance provides both positive and negative signals, which is required for stable offline preference optimization.
Each preference example consists of: (i) the original input prompt, (ii) the task-specific system prompt, (iii) a set of chosen responses, (iv) a set of rejected responses, where each response includes the full text, parsed answer, and reasoning trace, and (v) the ground truth label.

\textbf{Limitations of Data.}
Our training pipeline relies on datasets generated by a single teacher model, which introduces the risk of \textit{distributional homogenization}. 
This approach facilitates controlled, large-scale generation, but the resulting data may inherit specific reasoning artifacts and linguistic biases inherent to the teacher's architecture.
Though the synthetic data may inherit biases and reasoning patterns from the teacher model, the supervision/preference signal is defined solely by agreement with ground truth labels. As a result, the data provides a consistent target for correctness, which is sufficient for learning calibrated decision behavior, without making claims about the faithfulness or optimality of the generated reasoning traces.

All synthetic data is generated using only training splits of the respective benchmarks, and no validation or test labels are used during data construction.
For the datasets where the only existing set is split in the Train/Validation/Test set, we explicitly filter from the synthetic data any samples which also belong to the test set.

\subsection{Task formulation}

Following prior work, we evaluate confidence calibration using a multiple-choice formulation.
However, unlike \citet{xiao2025restoring}, we allow the model to generate an explicit chain-of-thought (CoT) prior to emitting the final label token.
Calibration metrics are computed using the predicted probability of the label token, independent of the preceding reasoning.

For all datasets, the order of the options is randomized to account for prompt length bias. 
In datasets where a single original split is randomly partitioned into training, validation, and test subsets, we generate supervised or preference training samples exclusively from the training subset.
To ensure strict separation between training and evaluation data, we explicitly remove from the training subset any examples that occur in the validation or test subsets for any random seed used in our experiments.

The validation samples are used to select the optimal training checkpoint, and the test samples are used to report accuracy and Expected Calibration Error (ECE).
\textsc{RewardBench~2} is evaluated on its full validation and test sets, which contain 187 examples each.

For all tasks, answer labels are canonicalized to a fixed vocabulary (e.g., ``A'', ``B'', ``entailment'', ``neutral'', ``contradiction'') prior to evaluation.
Instead of the regular options (A,B,C,D) of a MCQ, we formulate NLI tasks (ANLI,SNLI) to predict entailment, i.e. predicting "entailment", "neutral" or "contradiction" and calculating the ECE of the first token that makes up the words, removing any effect of bias added by label tokens from ECE calculation. 
For NLI tasks, confidence is computed using the probability assigned to the first token of the predicted label string.

During both training and evaluation, models are prompted to generate an explicit chain-of-thought prior to producing the final answer token.
For the cases where the model does not output a label, we implement a fallback of restrictive decoding such that we only consider the respective labels of each dataset.

\newpage

\section{Experimental Details}
\label{adx:appendix-experimental-details}
We perform our experiment on Qwen 3 4B \cite{yang2025qwen3}, on fixed seeds (0,1,2,3,4) and provide the average results in the main text. 
We use A100s 80GB for the experiments with resources from Unity, a collaborative, multi-institutional high-performance computing cluster managed by UMass Amherst Research Computing and Data. 
All experiments operate under a context window of $2048$ tokens, and we perform check pointing every $n$ steps (500 for supervised training and 100 for preference optimization) of training and retain the highest performing model at the end of training.
We use $\lambda=0.1$ for \projectnameshort and DPO+BCE.

For all datasets except \textsc{Reward Bench 2} (which has a test split of 187 samples), we uniformly sample 500 examples from the validation and test splits.
All samples are drawn uniformly at random without replacement.
For DPO + \projectnameshort, and DPO + BCE we subsample 500 preference pairs.

\subsection{Baselines}

\textbf{Supervised Finetuning} is the standard SFT baseline implemented via the base Transformer Library Trainer where the language model is finetuned on a custom instruction-following dataset. 
The objective is to minimize the cross-entropy loss:
\begin{equation}
    \mathcal{L}_{SFT}(\theta) = -\mathbb{E}_{(x, y) \sim \mathcal{D}_{SFT}} \left[ \sum_{t=1}^{|y|} \log P_\theta(y_t | x, y_{<t}) \right]
\end{equation}
This model serves as the base policy $\pi_{SFT}$ for all further preference optimization methods. 
The model is trained without any label smoothing for this baseline. 
We use a learning rate of $5\times10^{-5}$, a batch size of $2$, and train the model over 3 epochs over the entire training dataset. 

\textbf{Temperature Scaling} is the training free baseline which was explained in \citep{guo2017calibration} to address potential under/overconfidence. 
We implement Temperature scaling as a post-hoc calibration mechanism, which selects a valid temperature which minimizes the NLL loss on a given held out validation set.
The temperature ($\tau > 0$) is then used in standard multinomal sampling in LLMs:
\begin{equation}
    P_\theta(y_i | x; \tau) = \frac{\exp(z_i / \tau)}{\sum_{j} \exp(z_j / \tau)}
\end{equation}

\textbf{Label Smoothing} is a SFT alternative shown to improve confidence calibration in \citep{huang2025calibrated} where a small probability mass $\epsilon$ is distributed across the output distribution from the regular one-hot supervised labels normally used in Cross-Entropy. 
During the SFT phase, we apply Label Smoothing as a form of entropy regularization to prevent the model from over-fitting to the specific surface forms of the training data. 
The target distribution is modified by a factor $\epsilon$:
\begin{equation}
    y^{LS} = (1 - \epsilon)y^* + \frac{\epsilon}{K}
\end{equation}
where $\epsilon$ is the smoothing hyperparameter and $K$ is the vocabulary size. 
This regularization forces the model to maintain a "safety margin" in its logit space. 
For evaluation, this baseline tests whether a more conservative SFT distribution that acknowledges the existence of alternative valid completions.

\textbf{Direct Preference Optimization} \cite{rafailov2023direct} is the standard DPO objective, implemented via the DPOTrainer from TRL~\cite{vonwerra2022trl}, which aligns the model with human preferences without an intermediate reward model. 
DPO optimizes the policy $\pi_\theta$ directly on preference pairs $(x, y_w, y_l)$, under the Bradley-Terry model, by maximizing the log-likelihood of the preferred response relative to the rejected one, regularized by the KL-divergence from a reference policy $\pi_{ref}$:
\begin{equation}
    \mathcal{L}_{DPO}(\pi_\theta; \pi_{ref}) = -\mathbb{E}_{(x, y_w, y_l) \sim \mathcal{D}} \left[ \log \sigma \left( \beta \log \frac{\pi_\theta(y_w|x)}{\pi_{ref}(y_w|x)} - \beta \log \frac{\pi_\theta(y_l|x)}{\pi_{ref}(y_l|x)} \right) \right]
\end{equation}

\textbf{Regularized Calibration-Aware Finetuning} \cite{xiao2025restoring} is an expectation-maximization based algorithm which is performed on top of a DPO trained model. 
As models are fine-tuned for higher performance, they often enter a "non-calibratable regime" where achieving perfect calibration (ECE $=0$) is theoretically impossible given the model's accuracy threshold. 
To address this, \citealt{xiao2025restoring} proposed Regularized Calibration-Aware Finetuning (RCFT), which incorporates an explicit Expected Calibration Error (ECE) regularization term into the supervised objective. 
Specifically, the training objective minimizes a weighted combination of accuracy loss and calibration error:
\begin{equation}
    \mathcal{L}_{RCFT}(\theta) = \mathcal{L}_{SFT}(\theta) + \lambda \cdot \mathcal{L}_{ECE}(\theta)
\end{equation}
where $\lambda$ is a Lagrange multiplier (set to 1 for RCFT). 
The calibration term $\mathcal{L}_{ECE}$ is obtained using an EM-algorithm-based approach on a held out validation set which gives us an estimate of the calibration of a given model on a given dataset. 
We use a learning rate of $5\times10^{-6}$ which the authors use in the original work, with a batch size of $4$.

\textbf{DPO + $BCE(\widetilde z(x),c_{\theta}(x))$} is a baseline we implement on the basis of our proposed correctness surrogate. 
We utilize the Binary Cross-Entropy (BCE) loss as a stand in for mean-based estimators of a distribution between the predicted confidence of a given samples and $c_\theta(x)$ our proposed correctness surrogate $\widetilde z(x)$. 
This BCE objective is jointly optimized with the DPO objective similarly to our method on a per-token basis and aggregated over a sequence: 
\begin{equation}
    \mathcal{L}_{CAL}(\theta) = \mathcal{L}_{DPO}(\theta) + \lambda \cdot \mathcal{L}_{BCE}(\theta)
\end{equation}
where $\mathcal{L}_{BCE}(\theta)=\sum_{i=1}^s BCE(\widetilde z(x),c_{\theta}(x))$, $s$ is the length of the sequence $S$, $\lambda$ is a Lagrange multiplier (set to 0.1 for this baseline).

\subsection{Metrics}

We report two metrics for our experiments: (a) ECE and (b) Accuracy. Our task formulation allows the model to output a COT and then a categorical label for the Multiple Choice Question. 
This structure allows us to isolate the model's predicted confidence in its classification from the free-form generative process. 
To ensure a complete evaluation, we implement a fallback mechanism via constrained decoding: if a valid label is not detected in the free-form generation, we restrict decoding to the valid label set and force the model to emit a categorical token.
This mechanism is used solely to guarantee evaluation coverage and does not alter the model’s parameters or its confidence estimates for the selected label.
This ensures that every test sample receives a designated score and prevents evaluation bias stemming from formatting failures.

We note that ECE evaluates calibration at the level of the final label predictions and does not directly assess the correctness or faithfulness of the generated reasoning traces.

\textbf{ECE} \cite{xiao2025restoring} following prior work we partition the model's confidence scores into M equally spaced bins and calculate the weighted average difference between the bin's accuracy and its mean confidence: 
\begin{equation} 
\text{ECE} = \sum_{m=1}^{M} \frac{|B_m|}{N} | \text{acc}(B_m) - \text{conf}(B_m) | 
\end{equation}
In the presence of chain-of-thought prompting, this metric evaluates whether the model’s confidence in its final prediction remains aligned with empirical correctness, even when intermediate reasoning steps may be imperfect.
A lower ECE indicates better alignment between confidence and accuracy, which is often informally described as a model that “knows what it knows,” and is a critical requirement for deploying reward models in RLHF loops.
For our experiments we use $M=20$.

\textbf{Accuracy} We report the categorical accuracy as the primary measure of model capability. 
Given our MCQ formulation, accuracy is determined via exact string matching between the decoded label and the ground truth, after applying the fallback decoding mechanism when necessary.

\textbf{Confidence@k} We generate $k$ candidate answers for each sample, and pick the answer with the highest confidence. For MCQ tasks we consider the confidence of the label token similarly to ECE calculation.

\newpage

\section{Additional Results}
\label{adx:additional-results}
In this section we provide all experiments not present in the main text.

\subsection{In-distribution Variance}
\label{adx:appendix-var-in-distribution}
Here we report the variance of Accuracy and ECE metrics across 5 seeds, corresponding mean values reported in Table~\ref{tbl:in-distribution-main}
\begin{table*}[hbt!]
\caption{Variance of ECE and Accuracy across our method (DPO + \projectnameshort) and the baselines (Temperature Scaling, DPO, DPO + BCE, RCFT) for \textbf{in-distribution datasets} over 5 seeds. The first half of the table contains SFT baselines, the second half of the table has been initialized with the SFT without Label Smoothing baseline, and RCFT has been initialized with the DPO baseline. }
\vspace{-0.2cm}
\centering
\setlength{\tabcolsep}{5pt} 
\begin{tabular}{lcccccccccc}
\toprule
& \multicolumn{2}{c}{\textbf{Reward Bench 2}} & \multicolumn{2}{c}{\textbf{SNLI}} & \multicolumn{2}{c}{\textbf{ANLI}} & \multicolumn{2}{c}{\textbf{TLDR}} & \multicolumn{2}{c}{\textbf{CommonsenseQA}} \\
\cmidrule{2-11}
\textbf{Methods} & Acc. $\uparrow$ & ECE $\downarrow$ & Acc. $\uparrow$  & ECE $\downarrow$ & Acc. $\uparrow$  & ECE $\downarrow$ & Acc. $\uparrow$  & ECE $\downarrow$ & Acc. $\uparrow$  & ECE $\downarrow$ \\
 \midrule
SFT & 4.43 & 9.11 & 6.35 & 6.11 & 0.62 & 5.77 & 1.95 & 1.47 & 2.08 & 3.67 \\
+ Temperature Scaling & 4.81 & 2.34 & 3.31 & 2.43  & 4.37 & 1.81 & 2.12 & 1.88 & 2.52 & 2.40 \\
+ Label Smoothing & 6.51 & 3.68 & 3.18 & 3.51 & 2.36 & 1.01 & 2.51 & 5.05 & 3.68 & 4.99 \\
\midrule
+ DPO & 1.28 & 5.73 & 6.72 & 4.94 & 1.62 & 5.45 & 3.33 & 4.47 & 1.28 & 3.44 \\
+ RCFT & 2.55 & 5.25 & 2.53 & 5.42 & 1.78 & 4.35 & 2.40 & 3.34 & 3.33 & 6.97 \\
+ DPO + BCE & 4.15 & 2.76 & 5.28 & 2.09 & 1.39 & 5.65 & 3.62 & 3.14 & 0.82 & 1.86 \\
+ DPO + \projectnameshort (Ours) & 3.59 & 2.74 & 4.65 & 0.88 & 1.09 & 6.73 & 1.61 & 2.64 & 1.71 & 1.58 \\
\bottomrule
\end{tabular}
\label{tbl:in-distribution-var-adx}
\end{table*}

\subsection{Out-of-distribution Variance}
\label{adx:appendix-var-out-distribution}
Here we report the variance of Accuracy and ECE metrics across 5 seeds, corresponding mean values reported in Table~\ref{tbl:out-of-distribution-main}
\begin{table*}[hbt!]
\caption{Variance of ECE and Accuracy across our method (DPO + \projectnameshort) and the baselines (Temperature Scaling, DPO, DPO + BCE, RCFT) for \textbf{out-of-distribution datasets}, i.e. trained on MMLU and tested on other datasets. $\uparrow$ means higher values are better, and $\downarrow$ means lower values are better over 5 seeds. The first half of the table contains SFT baselines, the second half of the table has been initialized with the SFT without Label Smoothing baseline, and RCFT has been initialized with the DPO baseline.}
\vspace{-0.2cm}
\centering
\setlength{\tabcolsep}{5pt} 
\begin{tabular}{lcccccccccc}
\toprule
& \multicolumn{2}{c}{\textbf{Reward Bench 2}} & \multicolumn{2}{c}{\textbf{SNLI}} & \multicolumn{2}{c}{\textbf{ANLI}} & \multicolumn{2}{c}{\textbf{TLDR}} & \multicolumn{2}{c}{\textbf{CommonsenseQA}} \\
\cmidrule{2-11}
\textbf{Methods} & Acc. $\uparrow$ & ECE $\downarrow$ & Acc. $\uparrow$  & ECE $\downarrow$ & Acc. $\uparrow$  & ECE $\downarrow$ & Acc. $\uparrow$  & ECE $\downarrow$ & Acc. $\uparrow$  & ECE $\downarrow$ \\
 \midrule
SFT & 3.29 & 0.69 & 4.03 & 12.52 & 2.26 & 13.06 & 13.43 & 11.08 & 22.76 & 0.08 \\
+ Temperature Scaling & 4.51 & 2.83 & 11.07 & 0.82 & 6.51 & 2.04 & 3.95 & 5.20 & 9.91 & 2.84 \\
+ Label Smoothing & 1.66 & 1.73 & 3.83 & 9.56 & 0.84 & 1.43 & 10.04 & 5.70 & 24.18 & 2.18 \\
\midrule 
+ DPO & 5.18 & 9.36 & 2.96 & 10.82 & 0.98 & 11.07 & 12.86 &  17.36 & 2.21 & 18.24 \\
+ RCFT & 4.07 & 0.56 & 0.86 & 9.12 & 5.93 & 6.01 & 10.79 & 14.97 & 14.42 & 0.05 \\
+ DPO + BCE & 3.02 & 0.86 & 1.52 & 13.60 & 13.33 & 2.38 & 1.55 & 6.06 & 2.54 & 1.71 \\
+ DPO + \projectnameshort (ours) & 2.74 & 0.41 & 1.27 & 3.45 & 3.25 & 2.91 & 3.01 & 2.59 & 9.37 & 2.92 \\
\bottomrule
\end{tabular}
\label{tbl:out-of-distribution-var-adx}
\end{table*}

\newpage

\subsection{Confidence@k Variance}
\label{adx:appendix-var-conf-at-k}
Here we report the variance of Accuracy across 5 seeds and $k\in\{4,8\}$, corresponding mean values reported in Table~\ref{tbl:conf-at-k-main}
\begin{table*}[hbt!]
\caption{Variance Accuracy across our method (DPO + \projectnameshort) and the baselines (Temperature Scaling, DPO, DPO + BCE, RCFT) for \textbf{in-distribution datasets}. The first half of the table contains SFT baselines, the second half of the table has been initialized with the SFT without Label Smoothing baseline, and RCFT has been initialized with the DPO baseline.}
\vspace{-0.2cm}
\centering
\setlength{\tabcolsep}{5pt} 
\begin{tabular}{lcccccccccc}
\toprule
& \multicolumn{2}{c}{\textbf{Reward Bench 2}} & \multicolumn{2}{c}{\textbf{SNLI}} & \multicolumn{2}{c}{\textbf{ANLI}} & \multicolumn{2}{c}{\textbf{TLDR}} & \multicolumn{2}{c}{\textbf{CommonsenseQA}} \\
\cmidrule{2-11}
\textbf{Methods} & 4 & 8 & 4  & 8 & 4  & 8 & 4  & 8 & 4  & 8 \\
 \midrule
SFT & 3.10 & 1.92 & 4.45 & 2.81 & 0.41 & 0.26 & 1.22 & 0.78 & 1.31 & 0.84 \\
+ Label Smoothing & 4.52 & 2.87 & 2.18 & 1.42 & 1.41 & 0.89 & 1.44 & 0.93 & 2.37 & 1.51 \\
\midrule 
+ DPO & 0.83 & 0.51 & 4.58 & 2.96 & 0.91 & 0.58 & 1.84 & 1.19 & 0.82 & 0.54 \\
+ RCFT & 1.61 & 1.02 & 1.39 & 0.91 & 1.02 & 0.66 & 1.29 & 0.85 & 1.98 & 1.27 \\
+ RCFT & 1.61 & 1.02 & 1.39 & 0.91 & 1.02 & 0.66 & 1.29 & 0.85 & 1.98 & 1.27 \\
+ DPO + \projectnameshort (ours) & 2.12 & 1.34 & 2.48 & 1.57 & 0.62 & 0.40 & 0.97 & 0.64 & 0.91 & 0.59 \\
\bottomrule
\end{tabular}
\label{tbl:conf-at-k-var-adx}
\end{table*}

\subsection{Hyperparameters}
\label{adx:appendix-hyperparameters}

To ensure a fair comparison across preference optimization methods, we conduct a grid search over learning rates 
$\eta \in \{5\times10^{-6}, 5\times10^{-7}\}$, using the same search space for all methods in the in-distribution experiments.
For each method, we select the learning rate that minimizes the average Expected Calibration Error (ECE) across five independent random seeds, measured on the validation split of each dataset.
This selection criterion aligns with our primary objective of improving confidence calibration.
The out-of-distribution experiments for all datasets and preference optimization methods are run with a learning rate of $5\times10^{-6}$.

\begin{table}[hbt!]
\vspace{-0.2cm}\centering
\caption{Reward Bench 2 In-Distribution Learning Rates}
\begin{tabular}{lccccc}
\toprule
\textbf{Method} & 0 & 1 & 2 & 3 & 4 \\
\midrule
DPO & $5\times 10^{-6}$ & $5\times 10^{-6}$ & $5\times 10^{-6}$ & $5\times 10^{-6}$ & $5\times 10^{-6}$ \\
DPO + BCE & $5\times 10^{-6}$ & $5\times 10^{-7}$ & $5\times 10^{-7}$ & $5\times 10^{-6}$ & $5\times 10^{-6}$ \\
DPO + \projectnameshort & $5\times 10^{-6}$ & $5\times 10^{-6}$ & $5\times 10^{-6}$ & $5\times 10^{-6}$ & $5\times 10^{-6}$ \\
\bottomrule
\end{tabular}
\end{table}
\begin{table}[hbt!]
\vspace{-0.2cm}
\centering
\caption{SNLI In-Distribution Learning Rates}
\begin{tabular}{lccccc}
\toprule
\textbf{Method} & 0 & 1 & 2 & 3 & 4 \\
\midrule
DPO & $5\times 10^{-6}$ & $5\times 10^{-6}$ & $5\times 10^{-6}$ & $5\times 10^{-6}$ & $5\times 10^{-6}$ \\
DPO + BCE & $5\times 10^{-6}$ & $5\times 10^{-7}$  & $5\times 10^{-6}$ & $5\times 10^{-7}$  & $5\times 10^{-6}$ \\
DPO + \projectnameshort & $5\times 10^{-6}$ & $5\times 10^{-7}$  & $5\times 10^{-6}$ & $5\times 10^{-6}$ & $5\times 10^{-6}$ \\
\bottomrule
\end{tabular}
\end{table}

\begin{table}[hbt!]
\caption{ANLI In-Distribution Learning Rates}
\vspace{-0.2cm}
\centering
\begin{tabular}{lccccc}
\toprule
\textbf{Method} & 0 & 1 & 2 & 3 & 4 \\
\midrule
DPO & $5\times 10^{-7}$ & $5\times 10^{-6}$ & $5\times 10^{-6}$ & $5\times 10^{-7}$ & $5\times 10^{-6}$ \\
DPO + BCE & $5\times 10^{-6}$ & $5\times 10^{-7}$ & $5\times 10^{-7}$ & $5\times 10^{-7}$ & $5\times 10^{-7}$ \\
DPO + \projectnameshort & $5\times 10^{-7}$ & $5\times 10^{-7}$ & $5\times 10^{-6}$ & $5\times 10^{-6}$ & $5\times 10^{-6}$ \\
\bottomrule
\end{tabular}
\end{table}
\begin{table}[hbt!]
\caption{TLDR In-Distribution Learning Rates}
\vspace{-0.2cm}
\centering
\begin{tabular}{lccccc}
\toprule
\textbf{Method} & 0 & 1 & 2 & 3 & 4 \\
\midrule
DPO & $5\times 10^{-6}$ & $5\times 10^{-7}$ & $5\times 10^{-7}$ & $5\times 10^{-7}$ & $5\times 10^{-6}$ \\
DPO + BCE & $5\times 10^{-6}$ & $5\times 10^{-7}$ & $5\times 10^{-7}$ & $5\times 10^{-6}$ & $5\times 10^{-6}$ \\
DPO + \projectnameshort & $5\times 10^{-6}$ & $5\times 10^{-7}$ & $5\times 10^{-6}$ & $5\times 10^{-6}$ & $5\times 10^{-6}$ \\
\bottomrule
\end{tabular}
\end{table}

\begin{table}[hbt!]
\caption{Commonsense QA In-Distribution Learning Rates}
\vspace{-0.2cm}
\centering
\begin{tabular}{lccccc}
\toprule
\textbf{Method} & 0 & 1 & 2 & 3 & 4 \\
\midrule
DPO & $5\times 10^{-6}$ & $5\times 10^{-6}$ & $5\times 10^{-6}$ & $5\times 10^{-6}$ & $5\times 10^{-6}$ \\
DPO + BCE & $5\times 10^{-6}$ & $5\times 10^{-6}$ & $5\times 10^{-6}$ & $5\times 10^{-6}$ & $5\times 10^{-6}$ \\
DPO + \projectnameshort & $5\times 10^{-6}$ & $5\times 10^{-6}$ & $5\times 10^{-6}$ & $5\times 10^{-6}$ & $5\times 10^{-6}$ \\
\bottomrule
\end{tabular}
\end{table}

\newpage

\subsection{Time taken across training}
\label{adx:appendix-time-taken}

In this section we show the time taken to complete training across DPO, RCFT and DPO (DPO + \projectnameshort) in minutes. We note that we run RCFT for 5 epochs following the original work, and DPO for 6 epochs, while we find that DPO + \projectnameshort works best when we limit the training epochs and preference pairs participating in training.

\begin{table*}[hbt!]
\caption{Comparison of \textbf{training time} for our method (DPO + \projectnameshort) and baselines including DPO, and RCFT. 
We report the total training time \textbf{in minutes} on a single A100 (80GB) GPU.}
\vspace{-0.2cm}
\centering
\begin{tabular}{lccccc}
\toprule
\textbf{Methods} & \textbf{Reward Bench 2} & \textbf{SNLI} & \textbf{ANLI} & \textbf{TLDR} & \textbf{CommonsenseQA} \\
\midrule
DPO (6 epochs)  & 75 & 70 & 78 & 72 & 75 \\
RCFT (5 epochs) & 1410 & 1807 & 1606 & 2095 & 2909 \\
DPO + \projectnameshort (2 epochs)  & 135 & 76 & 104 & 124 & 98 \\
\bottomrule
\end{tabular}
\label{tbl:time-taken-main}
\end{table*}

\newpage

\section{Proofs and Derivations}
\label{adx:appendix-proofs-and-derivations}

\subsection{Formulation and Derivation of the Objective}

This section establishes the mathematical definitions and identities used to construct the \projectname.
Specifically, Section~\ref{subsec:decomposition-of-L1-risk} decomposes the population-level $L_1$ calibration error into a per-token objective.
Then Section~\ref{adx:appendix-absolute-deviation} derives the algebraic identity that transforms the non-differentiable absolute error into a linear, gradient-friendly function.
Finally Section~\ref{adx:appendix-surrogate} justifies using the probability margin as a monotonic surrogate for correctness based on multinomial logistic properties.

\subsubsection{Decomposition of $L_1$ risk.}
\label{subsec:decomposition-of-L1-risk}
In this section we provide an explicit decomposition of the $L_1$ risk to get the loss form we utilize in this paper.

\begin{align}
\mathbb{E}_{Z \sim \Pr(Z\mid x_t)}\!\left[ \frac{1}{T}\sum_{t=1}^{T} \bigl|c_\theta(x_t)-Z_t\bigr| \right]
&= \frac{1}{T}\sum_{t=1}^{T} \mathbb{E}_{Z_t \sim \Pr(Z\mid x_t)} \bigl[ \bigl|c_\theta(x_t)-Z_t\bigr| \bigr] \\
&= \frac{1}{T}\sum_{t=1}^{T} \sum_{z\in\{0,1\}} \Pr(Z_t=z\mid x_t)\, \bigl|c_\theta(x_t)-z\bigr|\\
&= \frac{1}{T}\sum_{t=1}^{T} \Big( \Pr(Z_t=1\mid x_t)\, |c_\theta(x_t)-1| + \Pr(Z_t=0\mid x_t)\, |c_\theta(x_t)-0| \Big) \\
&= \frac{1}{T}\sum_{t=1}^{T} \Big( z(x_t)\,(1-c_\theta(x_t)) + (1-z(x_t))\,c_\theta(x_t) \Big),
\label{eq:l1-decompose-adx}
\end{align}
Here $Z_t\in\{0,1\}$ is a Bernoulli random variable with success probability
$z(x_t)\coloneqq \Pr(Z_t=1\mid x_t)$, and we use the fact that
$c_\theta(x_t)\in[0,1]$ so that $|c_\theta(x_t)-1|=1-c_\theta(x_t)$ and
$|c_\theta(x_t)-0|=c_\theta(x_t)$.

The equality follows by expanding the expectation over the Bernoulli variable $Z_t$.

\subsubsection{Derivation of Absolute Difference Identity}
\label{adx:appendix-absolute-deviation}
Let $Z \in \{0,1\}$ denote correctness, i.e., $Z = 1$ if the prediction is correct and $0$ otherwise. Consider the absolute deviation between the model confidence $c_\theta(x) \in [0,1]$ and correctness:
\begin{equation}
|c_\theta(x) - Z|.
\end{equation}

Since $Z$ is binary, we can split the two cases:

\begin{itemize}
    \item If $Z = 1$: \quad $|c_\theta(x) - 1| = 1 - c_\theta(x)$.
    \item If $Z = 0$: \quad $|c_\theta(x) - 0| = c_\theta(x)$.
\end{itemize}

This can be compactly written as
\begin{equation}
|c_\theta(x) - Z| = Z \,(1 - c_\theta(x)) + (1-Z)\, c_\theta(x),
\end{equation}
which holds for both $Z=0$ and $Z=1$.

Taking the conditional expectation over $Z$ given $x$:
\begin{align}
\mathbb{E}[\,|c_\theta(x) - Z| \mid x\,] 
&= \mathbb{P}(Z=1 \mid x) \,(1 - c_\theta(x)) + \mathbb{P}(Z=0 \mid x) \, c_\theta(x) \\
&= z(x) \,(1 - c_\theta(x)) + (1 - z(x)) \, c_\theta(x),
\end{align}
where $z(x) = \mathbb{P}(Z=1 \mid x)$.

This form is convenient for constructing a differentiable surrogate loss, as it expresses the absolute deviation as a linear combination of $c_\theta(x)$ weighted by the probability of correctness.

\subsubsection{Surrogate of Correctness using Probability Margin}
\label{adx:appendix-surrogate}

In this section we justify usage of the surrogate correctness signal
\[
\widetilde z(x) \;=\; \sigma\!\left( {m(x)} \right), 
\qquad m(x)=p_{y^\star}(x)-p_{\bar y}(x),
\]
where \(p_{y^\star}(x)=\pi_\theta(y^\star\mid x)\) and 
\(p_{\bar y}(x)=\max_{y\neq y^\star}\pi_\theta(y\mid x)\).

Classical results for multinomial logistic models \cite{bartlett2008classification,tewari2007on}  show that  
\[
m(x_1)\ge m(x_2)
\quad\Longrightarrow\quad
z(x_1)\ge z(x_2),
\]
i.e. the probability that the top prediction is correct is monotone in the decision margin. Since the sigmoid is also strictly increasing, the surrogate satisfies the same ordering:
\[
m(x_1)\ge m(x_2)
\;\Longrightarrow\;
\widetilde z(x_1)\ge \widetilde z(x_2).
\]

For the logistic loss $\phi(v)=\log(1+\exp(-v))$, the conditional risk minimizer is given by
\[
f_\phi^*(\eta)=\log\frac{\eta}{1-\eta},
\]
as shown in \citet{zhang2004statistical}.  
Interpreting $f_\phi^*(\eta)$ as an optimal class score, we may introduce the notation
$\ell_y \coloneqq \phi(y f_\phi^*(p_y))$ to denote the loss evaluated at the optimal score
associated with class $y$.  
With this convention, the difference between the losses of the Bayes-optimal class
$y^\star$ and a competing class $\bar y$ can be written as
\[
\ell_{y^\star}-\ell_{\bar y}
= f_\phi^*(p_{y^\star})-f_\phi^*(p_{\bar y})
= \log\frac{p_{y^\star}}{p_{\bar y}}.
\]
When the probability gap $p_{y^\star}-p_{\bar y}$ is small, the log-ratio admits the local
expansion
\[
\log\frac{p_{y^\star}}{p_{\bar y}}
= (p_{y^\star}-p_{\bar y})
+ O\!\left((p_{y^\star}-p_{\bar y})^2\right),
\]
which follows from a first-order Taylor expansion of $\log(1+u)$ around $u=0$.

The probability margin $m(x) = p_{y^{\star}}(x) - p_{\bar{y}}(x)$ captures the gap in the model's certainty between the correct answer and its best alternative. When this gap is:
\begin{itemize}
    \item \textbf{Large and positive}: The model is significantly more confident in the correct token than any competitor, suggesting high likelihood of correctness. The sigmoid maps this to $\widetilde{z}(x) \approx 1$.
    \item \textbf{Near zero}: The model is torn between the correct token and an incorrect alternative, indicating uncertainty. The sigmoid yields $\widetilde{z}(x) \approx 0.5$.
    \item \textbf{Negative}: The model actually favors an incorrect token over the correct one, signaling likely error. The sigmoid produces $\widetilde{z}(x) < 0.5$, pushing confidence down during optimization.
\end{itemize}

The logistic function $\sigma(\cdot)$ naturally maps unbounded margins to the $[0,1]$ probability space while preserving ordering. 
Unlike a linear mapping, it provides diminishing returns for extreme margins (preventing overfitting to outliers) and steep gradients near zero (encouraging decisive improvements on uncertain predictions). 
This matches the Bayesian intuition that small changes in evidence should have maximal impact when uncertainty is highest.

\textbf{Remark on Robustness.} While the per-token calibration loss  $\mathcal{L}_{\mathrm{Cal}}$ is linear rather than a true $L_1$ loss, it inherits the key robustness property of $L_1$, i.e. \emph{uniformly bounded gradients}. 
As shown in Proposition~\ref{prop:l1_gradient_stability}, $|\partial \mathcal{L}_{\mathrm{Cal}}/\partial c_\theta(x)| \le 1$, meaning that  surrogate errors are never amplified during backpropagation. 
This mirrors the subgradient bound of $|c - z|$, where the influence of any single example is capped. 
Unlike $L_2$, where large discrepancies $(c - \widetilde{z})$ dominate updates, our loss ensures stable confidence training even when $\widetilde{z}(x)$ is far from the current $c_\theta(x)$.

The per-example calibration loss is
\[
\mathcal{L}(c;z) \;=\; z(1-c) + (1-z)c,
\]
whose unique minimizer is $c^\star=z$.  
When $\widetilde z(x)$ is used instead,
\[
\mathcal{L}(c; \widetilde z)
\;=\; \widetilde z(1-c) + (1-\widetilde z)c,
\]
and the minimizer is $c^\star=\widetilde z(x)$.
For L1 loss, the minimizer is the conditional median of the target.  
Since both $z(x)$ and $\widetilde z(x)$ are strictly increasing functions of the same statistic $m(x)$, they induce identical orderings and therefore identical medians.  
Thus minimizing the surrogate risk recovers the same minimizer as if the true $z(x)$ were known.

\begin{proposition}[Ordering-Preserving Surrogate]
\label{prop:ordering_preservation}
Minimizing the surrogate risk $\mathcal{L}(c; \widetilde{z})$ yields confidence scores 
that are order-preserving with respect to the true correctness probability. 
Since both $z(x)$ and $\widetilde{z}(x)$ are strictly increasing functions of $m(x)$, 
we have
\[
c^\star_{\text{true}}(x_1) \ge c^\star_{\text{true}}(x_2)
\iff
c^\star_{\text{surrogate}}(x_1) \ge c^\star_{\text{surrogate}}(x_2)
\]
for any inputs $x_1, x_2$. Thus, while the absolute confidence values may differ, 
the surrogate preserves the \emph{relative ranking} of model correctness.
\end{proposition}

\noindent
This follows from the fact that both $z(x)$ and $\widetilde{z}(x)$ are
strictly increasing functions of the same margin statistic.

The derivative of the surrogate is bounded:
\[
\biggl|\frac{\partial \widetilde z}{\partial m}\biggr|
= \sigma(m)(1-\sigma(m))
\le \frac{1}{4},
\]
ensuring stable gradients when added to preference optimization.

Therefore, the surrogate $\widetilde z(x)$ is justified because: 
(i) correctness probability is monotone in the decision margin,  
(ii) the surrogate is a logistic approximation of the Bayes posterior,  
(iii) it preserves the ordering of optimal confidence scores under the linear calibration loss, and  
(iv) its bounded derivative ($\le\frac{1}{4}$)  and the loss's bounded gradient ensure stable optimization.

\subsection{Statistical Properties and Robustness}

Here we analyze the noise properties of the surrogate and prove the robustness of the $L_1$ objective compared to $L_2$
Section~\ref{adx:appendix-l1-and-l2} formalizes the contamination model, defining the correctness surrogate under high-confidence as systematic outliers in the supervision signal.
Then Section~\ref{adx:appendix-surrogate-contamination} establishes that the $L_1$ loss satisfies symmetry conditions described in \citep{ghosh2017robust} required to tolerate target contamination.
Finally, Section~\ref{adx:appendix-l1-and-l2} demonstrates analytically that the $L_1$ objective remains stable under contamination, whereas the $L_2$ objective is biased by hallucinations.

\subsubsection{Bayes Optimality of Calibration Loss}
\label{adx:appendix-bayes-optimality-calibration}
In this section, we characterize Bayes-optimal solutions of common calibration losses under pointwise population minimization, first examining the Brier loss and then justifying the use of $L_1$ with a continuous surrogate for robust calibration.

We consider pointwise population risk minimization, i.e. for each input $x$, we study losses of the form $\mathbb{E}[\ell(c,Z)\mid x]$ and their minimizers with respect to the confidence $c\in[0,1]$.

The Brier loss ($L_2$ regression on binary outcomes) is a classical \textbf{strictly proper scoring rule}, it is uniquely minimized when the predicted confidence equals the true conditional probability $z(x)=\mathbb{P}(Z=1\mid x)$.

We include the following proposition and proof for completeness.

\begin{proposition}[Bayes optimality of the Brier ($L_2$) calibration loss]
\label{prop:brier-bayes}
For any input $x$, let $Z \in \{0,1\}$ denote the correctness indicator with
$z(x) = \mathbb{P}(Z=1 \mid x)$. The squared loss
\[
\mathcal{L}_{2}(c) := \mathbb{E}[(c - Z)^2 \mid x]
\]
is uniquely minimized at
\[
c^\star(x) = \mathbb{E}[Z \mid x] = z(x).
\]
\end{proposition}

\begin{proof}
Expanding the squared loss gives us the following
\[
\mathbb{E}[(c - Z)^2 \mid x] = (c - z(x))^2 + \mathrm{Var}(Z \mid x).
\]
The variance term here does not depend on $c$, hence minimization over $c$ yields
$c^\star(x) = z(x)$.
\end{proof}

\paragraph{Challenges with $L_1$ loss on binary outcomes.}
$L_1$ regression is often considered more robust than $L_2$ because its gradient magnitude is constant.  

However, it has a fundamental limitation for binary targets. The Bayes-optimal solution of 
$\mathbb{E}[|c-Y|\mid x]$ is any conditional median of $Y\mid x$.  

For $Z\sim\mathrm{Bernoulli}(z(x))$, the set of conditional medians is
\[
\mathrm{Med}(Z\mid x)=
\begin{cases}
\{0\}, & z(x)<\tfrac12,\\[2pt]
[0,1], & z(x)=\tfrac12,\\[2pt]
\{1\}, & z(x)>\tfrac12.
\end{cases}
\]

Consequently, $L_1$ regression cannot recover the underlying probability $z(x)$ except in the degenerate case $z(x)=\frac12$, causing predictions to collapse to 0 or 1.

\medskip

While $L_2$ is sensitive to outliers, $L_1$ offers a more stable objective due to its constant gradient. However, on binary outcomes, it only identifies the conditional median, which collapses to 0 or 1 and prevents fine-grained probability estimation.

Instead of estimating $z(x)$ directly, we replace the binary target $Z$ with a continuous surrogate $\tilde Z(x)\in[0,1]$ derived from the model’s logits, and we define calibration with respect to this surrogate.

We can now state Bayes-optimality of $L_1$ for the surrogate target.

\begin{proposition}[Bayes optimality of $L_1$ for deterministic surrogate targets]
\label{prop:l1-surrogate}
Let $\tilde Z(x)\in[0,1]$ be a deterministic continuous surrogate target
(e.g.\ the margin-based quantity defined in Eq.~(4)) given model logits.
Then the pointwise absolute loss
\[
\mathcal{L}_{1}(c) = |c - \tilde Z(x)|
\]
is uniquely minimized at $c^\star(x) = \tilde Z(x)$.
\end{proposition}

\begin{proof}
By definition of conditional median, $\operatorname{Median}(\tilde Z(x) \mid x)$ minimizes $\mathbb{E}[|c-\tilde Z(x)| \mid x]$. Uniqueness follows when $\tilde Z(x) \mid x$ has a density (ensured by the sigmoid transformation of continuous probabilities).
\end{proof}

By replacing the binary target with a continuous surrogate, the $L_1$ loss’s median-seeking behavior aligns with $\tilde Z(x)$, providing a stable and noise-tolerant calibration objective.

\subsubsection{Surrogate as a Contaminated Observation}
\label{adx:appendix-surrogate-contamination}
We formalize the connection between our surrogate-based calibration loss and the robust loss framework of \citet{ghosh2017robust}.

\begin{definition}[Surrogate Contamination Model]
\label{def:contamination}
For each input $x$, let $Z \in \{0,1\}$ be the true correctness indicator. The surrogate signal $\widetilde{z}(x)$ is generated as:
\begin{equation}
\widetilde{z}(x) = Z + \epsilon(x)
\end{equation}
where the noise term $\epsilon(x)$ satisfies:
\begin{enumerate}
    \item \textbf{Boundedness}: $\epsilon(x) \in [-Z, 1 - Z]$ almost surely. Equivalently, $\epsilon\in[0,1]$ when $Z=0$ (hallucination) and $\epsilon\in[-1.0]$ when $Z=1$, (since $\widetilde{z}(x) \in [0, 1]$).
    \item \textbf{Sign consistency}: With probability at least $1-\alpha_x$, the surrogate preserves correctness ordering, i.e., $\operatorname{sgn}(\widetilde{z}(x) - 0.5) = \operatorname{sgn}(Z - 0.5)$. This assumption is not required for the robustness result below but clarifies the intended regime where the surrogate is informative on average.
    \item \textbf{Contamination}: Let $\alpha_x  \mathbb{P}\big( \operatorname{sgn}(\widetilde{z}(x) - 0.5) \neq \operatorname{sgn}(Z - 0.5) \big)$ denote the probability of sign-flipping contamination, with $\alpha_x \in [0, 0.5)$. Under this definition, contamination implies $|\epsilon(x)| \ge 0.5$.
\end{enumerate}
\end{definition}

\textbf{Global Contamination Bound}: For the robustness guarantees to hold, we require $\alpha := \sup_{x \in \mathcal{X}} \alpha_x < 0.5$ (worst-case) or $\alpha := \mathbb{E}_x[\alpha_x] < 0.5$ (average-case). This ensures the population-level contamination rate remains below the breakdown point of the median.

The surrogate calibration loss is defined as the expected $L_1$ loss:
\begin{equation}
\mathcal{L}_{\text{cal}}(c_\theta, \widetilde{z}) = \mathbb{E}_{x}\big[\widetilde{z}(x)(1 - c_\theta(x)) + (1 - \widetilde{z}(x))c_\theta(x)\big].
\end{equation}
For fixed $x$, the pointwise loss is $\mathcal{L}_{\text{cal}}(c_\theta(x), \widetilde{z}(x)) = \widetilde{z}(x)(1 - c_\theta(x)) + (1 - \widetilde{z}(x))c_\theta(x)$.

Definition~\ref{def:contamination} should be viewed as a continuous-valued analogue of the discrete label contamination models studied in \citet{ghosh2017robust}, extending their framework from binary labels to surrogate correctness signals in $[0,1]$.

\paragraph{Remark on Systematic Noise and the Weak Accuracy Assumption.}
We acknowledge that in the context of generative LLMs, the error term $\epsilon(x)$ is not strictly random or symmetric; hallucinations induce systematic error where the model is confident ($\tilde{Z} \to 1$) yet incorrect ($Z=0$). However, as shown in \citet{ghosh2017robust}, robustness under contamination follows from symmetry of the loss rather than symmetry of the noise, and the breakdown point of the median ensures stability provided the contamination rate remains below $0.5$. 
Our formulation holds under a \textit{Weak Accuracy Assumption}: the surrogate contamination rate is bounded by $\alpha < 0.5$, where $\alpha = \sup_x \alpha_x$ or $\alpha = \mathbb{E}_x[\alpha_x]$. This assumption concerns the quality of the correctness signal, not the model's intrinsic accuracy, i.e. it requires that for any context, the surrogate correctly indicates the true label direction more often than not. 
Under this condition, the conditional median remains anchored to the true correctness $Z$. 
Unlike mean-based estimators (e.g., $L_2$), the $L_1$ objective exhibits bounded influence, ensuring that high-confidence hallucinations cannot arbitrarily skew the calibration target.

\begin{proposition}[Symmetry under Contamination]
\label{prop:symmetry_contam}
Under Definition~\ref{def:contamination}, the calibration loss $\mathcal{L}_{\text{cal}}(c_\theta, \widetilde{z})$ satisfies:
\begin{equation}
\mathcal{L}_{\text{cal}}(c_\theta, \widetilde{z}) + \mathcal{L}_{\text{cal}}(c_\theta, 1 - \widetilde{z}) = 1
\end{equation}
for any $c_\theta(x) \in [0, 1]$. This is the continuous analogue of the loss symmetry condition in \citet{ghosh2017robust} (Eq.~2) for binary classification.
\end{proposition}

\begin{proof}
By direct substitution from the definition of $\mathcal{L}_{\text{cal}}$:
\begin{align*}
\mathcal{L}_{\text{cal}}(c_\theta, \widetilde{z}) + \mathcal{L}_{\text{cal}}(c_\theta, 1 - \widetilde{z}) 
&= \left[\widetilde{z}(1 - c_\theta) + (1 - \widetilde{z})c_\theta\right] \\
&\quad + \left[(1 - \widetilde{z})(1 - c_\theta) + \widetilde{z}c_\theta\right] \\
&= \widetilde{z} - \widetilde{z}c_\theta + c_\theta - \widetilde{z}c_\theta + 1 - \widetilde{z} - c_\theta + \widetilde{z}c_\theta + \widetilde{z}c_\theta \\
&= 1.
\end{align*}
Thus the loss is symmetric in the sense required by \citet{ghosh2017robust}.
\end{proof}

Importantly, the symmetry condition of \citet{ghosh2017robust} applies to the loss, not to the noise distribution. Therefore, systematic contamination (e.g., high-confidence hallucinations where $\epsilon(x) \to 1$ while $Z=0$) is tolerated in the population risk provided the \textit{sign-flipping contamination rate} $\alpha$ remains below $0.5$, consistent with the robustness guarantees of \citet{ghosh2017robust}.

\subsubsection{Robustness of L1 risk over L2 risk with a noisy surrogate}
\label{adx:appendix-l1-and-l2}
We provide a formal justification for preferring the $L_1$ risk over the $L_2$ risk in our setting.  
Because we employ a noisy but differentiable approximation of accuracy to explicitly optimize for calibration, the $L_2$ risk introduces an implicit bias due to the long-tailed noise distribution.  We demonstrate this effect analytically below.
This analysis adapts the contamination model from \citeauthor{ghosh2017robust} to our setting where the 'clean signal' is the true correctness probability $z(x)$ and the 'noise' is the surrogate error $\epsilon(x)$ described in Equation~\ref{eq:cont-model}.

We first present a population-level analysis to illustrate the inherent robustness of the $L_1$ calibration loss compared to $L_2$ under surrogate contamination.
This analysis provides a clear conceptual understanding of the mechanism. 
We then complement this with a finite-sample analysis, showing that the robustness persists with high probability when training on a limited number of samples.

\textbf{{Population-Level Analysis}}

We consider the following setting that isolates the effect of noise in the surrogate.  Let \(z\in\mathbb{R}\) denote the true target of interest (in calibration contexts \(z\in[0,1]\)), and suppose we observe a noisy surrogate \(S\) intended to approximate \(z\).  We compare the population-optimal constant predictors obtained by minimizing the squared (squared-loss / \(L_2\)) and absolute (absolute-loss / \(L_1\)) objectives
w.r.t.\ the surrogate \(S\):
\[
c_{L2} \;=\; \arg\min_c \mathbb{E}[(c-S)^2] = \mathbb{E}[S], 
\qquad
c_{L1} \;=\; \arg\min_c \mathbb{E}[|c-S|] \in \operatorname{median}(S).
\]

The surrogate is distributed as
\begin{equation}\label{eq:cont-model}
S \;\sim\; (1-\alpha)\,F_0(\cdot - z) \;+\; \alpha\,\delta_{z+M},
\end{equation}
where $\delta_{z+M}$ is a Dirac delta representing a point mass at $z+M$, i.e., a potential outlier. With probability $1-\alpha$ the surrogate lies in the clean component around $z$, and with probability $\alpha$ it is replaced by the outlier $z+M$.

\begin{theorem}[Robustness under contamination]
\label{thm:robust-contamination}
Under model~\eqref{eq:cont-model}, suppose \(F_0\) has median \(0\) and mean \(\mu_0=\mathbb{E}_{F_0}[\varepsilon]\) (if the mean exists). Then

\begin{enumerate}
\item \(c_{L2} = \mathbb{E}[S] = z + (1-\alpha)\mu_0 + \alpha M\). In particular,
    if \(\mu_0=0\) then \(c_{L2}=z+\alpha M\).
\item If $\alpha<1/2$, any population median $c_{L1}$ satisfies
\[
(1-\alpha)F_0(c_{L1}-z) = 1/2.
\] 
    Thus, \(c_{L1} = z + \Delta_\alpha\), where \(\Delta_\alpha\) depends only on \(\alpha\) and \(F_0\) and is \textbf{independent} of the outlier magnitude \(M\).
\end{enumerate}

Consequently, for any fixed \(\alpha\in(0,1/2)\) and arbitrarily large \(|M|\), \(|c_{L2}-z| \approx \alpha|M|\) grows linearly with the outlier magnitude, while \(|c_{L1}-z| = |\Delta_\alpha|\) remains constant. Thus \(c_{L1}\) is strictly more robust to the outlier-contamination described by~\eqref{eq:cont-model}.
\end{theorem}

\begin{proof}
Part (1) follows from the law of total expectation:
\[
\mathbb{E}[S] = (1-\alpha)\mathbb{E}_{F_0}[z+\varepsilon] + \alpha(z+M)
= z + (1-\alpha)\mu_0 + \alpha M.
\]
For part (2): The cumulative distribution function of \(S\) is \(F_S(x) = (1-\alpha)F_0(x-z)\) for \(x < z+M\). 
Since $\alpha < 1/2$ and the point mass at $z+M$ lies strictly outside the support of the clean component for sufficiently large M, the jump at $z+M$ has size $\alpha<1/2$, so the median must occur within the clean component.
We solve for the median \(c_{L1}\) by setting the CDF to \(1/2\):
\[
(1-\alpha)F_0(c_{L1} - z) = 1/2 \implies F_0(c_{L1} - z) = \frac{1}{2(1-\alpha)}.
\]
Since \(F_0\) is independent of \(M\), the offset \(\Delta_\alpha = c_{L1}-z\) is constant with respect to \(M\).
\end{proof}

Consider the usual squared true loss \(\mathcal{R}_2(c)=(c-z)^2\). Under the contamination model with \(\mu_0=0\),
\[
\mathcal{R}_2(c_{L2}) = (\alpha M)^2, \qquad \mathcal{R}_2(c_{L1}) = \Delta_\alpha^2
\quad(\text{constant w.r.t } M).
\]
Thus for any fixed \(\alpha\in(0,1/2)\) and sufficiently large \(|M|\), the L\(_1\)-trained solution attains strictly smaller true squared risk than the L\(_2\)-trained solution.

\textbf{Finite-Sample Analysis}

We now consider the finite-sample case. Let $S_1,\dots,S_n$ be i.i.d.\ draws from~\eqref{eq:cont-model}, and define the empirical estimators:
\[
\hat c_{n,2} = \frac{1}{n}\sum_{i=1}^n S_i, \qquad
\hat c_{n,1} = \mathrm{median}(S_1,\dots,S_n).
\]
These are the finite-sample analogues of the population-optimal constants, i.e. $\hat c_{n,2}$ is the empirical mean (an unbiased estimator of $c_{L2}$) and $\hat c_{n,1}$ is the empirical median (a consistent estimator of $c_{L1}$).

Assume $|\varepsilon|\le B$ almost surely for $\varepsilon\sim F_0$.

\begin{lemma}[Finite-sample bias of the sample mean]
\label{lem:meanbias}
Under the above model, for any $\tau\in(0,1/2-\alpha)$, with probability at least $1-2\exp(-2\tau^2 n)$,
\[
\hat c_{n,2} \in z + \big[(\alpha-\tau)M - B,\; (\alpha+\tau)M + B\big],
\]
and therefore, if $(\alpha-\tau)|M| > B$,
\[
|\hat c_{n,2}-z| \ge (\alpha-\tau)|M| - B.
\]
\end{lemma}

\begin{proof}
Let $K \sim \operatorname{Binomial}(n,\alpha)$ denote the random number of contaminated samples.
By Hoeffding's inequality, $K/n \in [\alpha-\tau,\alpha+\tau]$ with probability at least $1-2\exp(-2\tau^2 n)$. Conditional on $K$, $\hat c_{n,2}=z+(K/n)M+(1/n)\sum_{\text{clean}}\varepsilon_i$. Since $|\varepsilon_i|\le B$, the result follows.
\end{proof}

\begin{lemma}[Finite-sample stability of the sample median]
\label{lem:medianstability}
Let $K$ be the number of outliers such that $S_i = z+M$. If $K < n/2$, then the empirical median $\hat c_{n,1}$ is bounded by the range of the clean samples. Specifically, if the clean noise satisfies $|\varepsilon| \le B$, then:
\[
\hat c_{n,1} \in [z-B, z+B], \quad \text{implying} \quad |\hat c_{n,1} - z| \le B.
\]
This bound holds regardless of the magnitude of $M$.
\end{lemma}

\begin{proof}
The empirical median is the $(\frac{n+1}{2})$-th order statistic (or the average of the two middle statistics). If $K < n/2$ samples are contaminated with $z+M$ (where $M$ is large), then at least $\lceil (n+1)/2 \rceil$ samples remain in the clean set $\{z+\varepsilon_i\}$. 

The middle statistic(s) must be drawn from the clean set. 
Since every clean sample $S_i$ satisfies $z-B \le S_i \le z+B$, any value (or average) derived from them must also fall within $[z-B, z+B]$.
\end{proof}

Assume $F_0$ has a density $f_0$ with $f_0(0)\ge f_{\min}>0$. Using the Dvoretzky--Kiefer--Wolfowitz inequality,
\[
\Pr\big( \|F_n - F\|_\infty > \varepsilon \big) \le 2e^{-2n\varepsilon^2},
\]
we obtain the following.

\begin{proposition}[Median concentration via DKW]
\label{prop:DKW}
Under the assumptions of Equation~\ref{eq:cont-model} with $\alpha<1/2$, let $c_{L1}$ be the true population median. Then for any $\rho > 1/n$,,
\[
\Pr\left(|\hat c_{n,1}-c_{L1}| > \frac{\rho}{(1-\alpha)f_{\min}}\right) \le 2e^{-2n(\rho - 1/n)^2}.
\]
\end{proposition}

\begin{proof}
We establish the concentration bound through three steps.

\bigskip
The Dvoretzky--Kiefer--Wolfowitz inequality states that for any $\delta > 0$,
\[
\Pr\big(\sup_{t \in \mathbb{R}} |F_n(t) - F(t)| > \delta\big) \le 2e^{-2n\delta^2}.
\]
Hence, with probability at least $1 - 2e^{-2n\delta^2}$,
\begin{align}
\sup_{t \in \mathbb{R}} |F_n(t) - F(t)| \le \delta. 
\label{eq:dwk-1}    
\end{align}

\bigskip
Let $c_{L1}$ be the population median. Applying the bound from Eq.~\ref{eq:dwk-1} at $t = \hat c_{n,1}$ gives
\[
|F(\hat c_{n,1}) - F(c_{L1})| \le |F(\hat c_{n,1}) - F_n(\hat c_{n,1})| + |F_n(\hat c_{n,1}) - 1/2| \le \delta + 1/n.
\]
Here we used that $F_n(\hat c_{n,1}) \in [1/2, 1/2 + 1/n]$ by definition of the empirical median and $F(c_{L1}) = 1/2$.

\bigskip
Because $\alpha < 1/2$, the median lies in the clean component's support. The mixture density is $f(u) = (1-\alpha)f_0(u-z)$. Since $f_0(0) \ge f_{\min}$, the density near $c_{L1}$ is lower-bounded by:
\[
f(u) \ge (1-\alpha)f_{\min}.
\]
By the mean value theorem, $|F(\hat c_{n,1}) - F(c_{L1})| = f(c^*) |\hat c_{n,1} - c_{L1}|$ for some $c^*$ between $\hat c_{n,1}$ and $c_{L1}$. Combining this with the previous inequality:
\[
(1-\alpha)f_{\min} |\hat c_{n,1} - c_{L1}| \le |F(\hat c_{n,1}) - F(c_{L1})| \le \delta + 1/n.
\]
Rearranging yields the result.
\end{proof}

Combining Lemmas~\ref{lem:meanbias} and~\ref{lem:medianstability} with Proposition~\ref{prop:DKW}, we find that for a fixed contamination rate $\alpha \in (0, 1/2)$, with probability at least $1-2\exp(-2\tau^2 n) - \exp(-2(\tfrac{1}{2}-\alpha)^2 n) - 2e^{-2n\rho^2}$, as $|M|$ grows:
\[
|\hat c_{n,2}-z| \ge (\alpha-\tau)|M| - B,
\]
\[
|\hat c_{n,1}-z| \le \underbrace{\frac{\rho}{(1-\alpha)f_{\min}}}_{\text{estimation error}} + \underbrace{|\Delta_\alpha|}_{\text{population bias}}.
\]
Choosing $\tau = \alpha/2$ and $\rho = n^{-1/2}$, for sufficiently large $|M| \gg B/\alpha$, the $L_1$ estimator's error is dominated by $O(n^{-1/2}) + |\Delta_\alpha|$, while the $L_2$ estimator's error grows as $\Omega(\alpha|M|)$.
This comparison reveals the fundamental advantage of the $L_1$ risk in our setting, i.e. the bias of the $L_2$ estimator ($\hat c_{n,2}$) grows linearly with the magnitude of surrogate noise $M$, whereas the bias of the $L_1$ estimator ($\hat c_{n,1}$) is bounded by a constant $\Delta_\alpha$ that is independent of $M$.

We note the following considerations:

\begin{enumerate}
    \item If $\alpha\ge1/2$, i.e. if contamination exceeds half the samples, the median may shift to the outlier point $z+M$. In that regime both mean and median can be arbitrarily far from (z).
    \item When $|M|$ is comparable to the clean noise spread, the bias $\alpha M$ is negligible and the two estimators perform similarly.
\end{enumerate}

\subsection{Optimization Dynamics and Stability}

Here, we provide properties regarding the stability and convergence of the objective during gradient-based optimization.
Section~\ref{adx:appendix-gradient-stability} proves that the gradient of the loss with respect to the confidence score is strictly bounded, preventing unstable updates.
Then Section~\ref{adx:appendix-bounded-log-prob-gradients} shows that the gradients backpropagated to the model's logits are uniformly bounded, ensuring the calibration term does not overpower preference learning.
Section~\ref{adx:appendix-convergence-properties} establishes that stochastic subgradient methods converge despite the non-smoothness introduced by the max operator in the surrogate.
Finally Section~\ref{adx:ordering-stability-proof} proves that, under specific bounds, adding the calibration loss does not flip the relative preference ordering between chosen and rejected responses.

\subsubsection{Gradient Stability of L1 loss}
\label{adx:appendix-gradient-stability}
\begin{proposition}[Gradient Stability]
\label{prop:l1_gradient_stability}
The subgradients of the per-token calibration loss with respect to
$c_\theta(x)$ satisfy
\begin{equation}
\Bigg|\frac{\partial \mathcal{L}_{\mathrm{Cal}}}{\partial c_\theta(x)}\Bigg|
= \bigl|1 - 2\widetilde{z}(x)\bigr| \le 1.
\end{equation}
In contrast, the gradient of the $L_2$ calibration loss scales as
$2\bigl(c_\theta(x) - \widetilde{z}(x)\bigr)$, and therefore grows linearly with
the surrogate error.
\end{proposition}

\begin{proof}
Recall the per-token calibration loss
\begin{equation}
\mathcal{L}_{\mathrm{Cal}}\bigl(c_\theta(x), \widetilde{z}(x)\bigr)
=
\widetilde{z}(x)\bigl(1 - c_\theta(x)\bigr)
+
\bigl(1 - \widetilde{z}(x)\bigr)c_\theta(x),
\end{equation}
where $c_\theta(x)\in[0,1]$ and $\widetilde{z}(x)\in[0,1]$.

Expanding terms yields
\begin{align}
\mathcal{L}_{\mathrm{Cal}}
&= \widetilde{z}(x) - \widetilde{z}(x)c_\theta(x)
+ c_\theta(x) - \widetilde{z}(x)c_\theta(x) \\
&= \widetilde{z}(x) + c_\theta(x)\bigl(1 - 2\widetilde{z}(x)\bigr).
\end{align}

The loss is affine in $c_\theta(x)$ and hence differentiable everywhere on
$(0,1)$ with constant derivative
\begin{equation}
\frac{\partial \mathcal{L}_{\mathrm{Cal}}}{\partial c_\theta(x)}
=
1 - 2\widetilde{z}(x).
\end{equation}
Because $\widetilde{z}(x)\in[0,1]$, we have
\begin{equation}
-1 \le 1 - 2\widetilde{z}(x) \le 1,
\end{equation}
which implies
\begin{equation}
\Bigg|\frac{\partial \mathcal{L}_{\mathrm{Cal}}}{\partial c_\theta(x)}\Bigg|
\le 1.
\end{equation}

For comparison, the squared calibration loss
\begin{equation}
\mathcal{L}_2\bigl(c_\theta(x), \widetilde{z}(x)\bigr)
=
\bigl(c_\theta(x) - \widetilde{z}(x)\bigr)^2
\end{equation}
has gradient
\begin{equation}
\frac{\partial \mathcal{L}_2}{\partial c_\theta(x)}
=
2\bigl(c_\theta(x) - \widetilde{z}(x)\bigr),
\end{equation}
which grows linearly with the discrepancy between $c_\theta(x)$ and
$\widetilde{z}(x)$.

Thus, \projectnameshort yields uniformly bounded gradients, ensuring
stable confidence updates even in the presence of large surrogate errors.
\end{proof}

\subsubsection{Bounded Log-Probability Gradients of \projectnameshort}
\label{adx:appendix-bounded-log-prob-gradients}
\begin{proposition}[Bounded Log-Probability Gradients of \projectnameshort]
\label{prop:cal-logprob-bound}
Let $\mathcal{L}_{\mathrm{Cal}}(x;\theta) = \ell(c) = \widetilde{z}(1-c) + (1-\widetilde{z})c$ , 
where $c_\theta(x) = \max_{y} \pi_\theta(y \mid x)$ denotes the model confidence
and $\tilde Z(x) \in [0,1]$ is a fixed surrogate target.
Then, for any output token $y$,
the subgradient of $\mathcal{L}_{\mathrm{Cal}}$ with respect to the log-probability
$\log \pi_\theta(y \mid x)$ satisfies
\[
\Big|\tfrac{\partial \mathcal{L}_{\mathrm{Cal}}}{\partial \log \pi_\theta(y \mid x)}\Big|
\;\le\; G(x),
\]
where
\[
G(x) \;:=\; c_\theta(x)\big(1 - c_\theta(x)\big) \;\le\; \tfrac{1}{4}.
\]
Consequently, the gradient magnitude is uniformly bounded and independent of
the surrogate noise in $\tilde Z(x)$.
\end{proposition}

\begin{proof}
We apply the chain rule to the composition
$\mathcal{L}_{\mathrm{Cal}}(x;\theta)
= \ell(c_\theta(x))$,
where $\ell(c) = \widetilde{z}(1-c) + (1-\widetilde{z})c$.

The function $\ell(c)$ is convex and its subgradient satisfies
\[
\Big|\tfrac{\partial \ell}{\partial c}\Big| \le 1
\quad \text{for all } c \in [0,1].
\]
For the per-token loss in \eqref{eq:per-token-cal-loss}, $\ell(c) = \widetilde{z}(1-c) + (1-\widetilde{z})c$ is linear in $c$ with gradient $\tfrac{\partial \ell}{\partial c} = 1-2\widetilde{z}$, so $|\tfrac{\partial \ell}{\partial c}| \le 1$ holds.

Let $y^\star = \arg\max_y \pi_\theta(y\mid x)$.
Then $c_\theta(x) = \pi_\theta(y^\star\mid x)$.
Using standard softmax identities,
\[
\tfrac{\partial \pi_\theta(y^\star\mid x)}{\partial \log \pi_\theta(y\mid x)}
=
\begin{cases}
\pi_\theta(y^\star\mid x)\big(1 - \pi_\theta(y^\star\mid x)\big),
& y = y^\star, \\[4pt]
-\pi_\theta(y^\star\mid x)\pi_\theta(y\mid x),
& y \neq y^\star .
\end{cases}
\]
In both cases, the magnitude is upper bounded by
$c_\theta(x)(1 - c_\theta(x))$.

Combining the above,
\[
\Big|\tfrac{\partial \mathcal{L}_{\mathrm{Cal}}}{\partial \log \pi_\theta(y \mid x)}\Big|
=
\Big|\tfrac{\partial \ell}{\partial c}\Big|
\cdot
\Big|\tfrac{\partial c_\theta(x)}{\partial \log \pi_\theta(y \mid x)}\Big|
\;\le\;
c_\theta(x)(1 - c_\theta(x)).
\]

Finally, since $c(1-c)$ is maximized at $c=\tfrac{1}{2}$,
we obtain the uniform bound $G(x) \le \tfrac{1}{4}$.
\end{proof}

\subsubsection{Convergence Properties of the Calibration Loss under Token Switching}
\label{adx:appendix-convergence-properties}
This appendix formalizes why optimization of the per-token calibration loss in Eq. \eqref{eq:per-token-cal-loss} admits standard convergence guarantees despite the presence of non-smooth token switching induced by the $\max$ operator in the surrogate $\widetilde{z}$.

\paragraph{Local Lipschitzness of Transformer Probabilities}
The confidence $c_\theta(x_t)$ is defined as the maximum predicted probability: $c_\theta(x_t) = \max_y \pi_\theta(y \mid x_t)$. This is locally Lipschitz with constant $L_\pi$ by the same argument as for $p_{\bar y}(x_t)$.

Modern transformer architectures are compositions of affine maps, softmax operations, and smooth nonlinearities. For any fixed input $x_t$, the mapping
$\theta \mapsto \pi_\theta(y \mid x_t)$ is locally Lipschitz continuous in
$\theta$. Consequently, for every $\theta$ there exists a neighborhood
$\mathcal{U}$ and a constant $L_\pi>0$ such that
\begin{equation}
\label{eq:pi-lipschitz}
\bigl| \pi_{\theta'}(y \mid x_t) - \pi_{\theta}(y \mid x_t) \bigr|
\;\le\; L_\pi \, \|\theta' - \theta\|,
\qquad \forall \theta',\theta \in \mathcal{U}.
\end{equation}

The operator
\begin{equation}
p_{\bar y}(x_t) \;=\; \max_{y \neq y_t^\star} \pi_\theta(y \mid x_t)
\end{equation} 
is $1$-Lipschitz as a function of the probability vector, and therefore locally Lipschitz in $\theta$. Since the sigmoid function $\sigma(\cdot)$ is smooth and globally Lipschitz, the surrogate correctness signal $\widetilde{z}(x_t) = \sigma(p_{y^\star}(x_t) - p_{\bar y}(x_t))$ is also locally Lipschitz in $\theta$.

It follows that the per-token calibration loss is locally Lipschitz:
\begin{equation}
\label{eq:loss-lipschitz}
\bigl| \mathcal{L}_{\mathrm{Cal}}(\theta') - \mathcal{L}_{\mathrm{Cal}}(\theta) \bigr|
\;\le\; L_{\mathcal{L}} \, \|\theta' - \theta\|
\end{equation}
for all $\theta',\theta$ in a sufficiently small neighborhood.

\paragraph{Effect of Token Switching on the Loss}
Token switching corresponds to changes in the identity of the maximizer in $p_{\bar y}(x_t)$, which occurs only when two or more non-ground-truth tokens have nearly equal probabilities. While such switching events may introduce nondifferentiability in the gradient, they do not create discontinuities in the loss value itself.

\paragraph{Subgradient Descent Progress}
Since $\mathcal{L}_{\mathrm{Cal}}$ is locally Lipschitz on $\Theta$, the Clarke subgradient satisfies the standard inequality \cite{davis2020stochastic}:
\begin{equation}
\mathbb{E}\bigl[\mathcal{L}_{\mathrm{Cal}}(\theta_{k+1})\bigr]
\;\le\;
\mathbb{E}\bigl[\mathcal{L}_{\mathrm{Cal}}(\theta_k)\bigr]
- \eta_k \, \mathbb{E}\bigl[\min_{g \in \partial \mathcal{L}_{\mathrm{Cal}}(\theta_k)} \|g\|^2\bigr]
+ L_{\mathcal{L}} \eta_k^2 \sigma^2,\end{equation}
Critically, this holds even at non-differentiable points because the Clarke subdifferential exists everywhere and the stochastic oracle is unbiased. 
The quadratic remainder term $O(\eta_k^2)$ is sufficient for convergence under appropriate step-size conditions.

\paragraph{Stochastic Oracle Assumptions}
We assume access to an unbiased stochastic subgradient oracle $g(\theta; \xi)$ with bounded variance:
\begin{equation}
\mathbb{E}_\xi[g(\theta; \xi)] \in \partial \mathcal{L}_{\mathrm{Cal}}(\theta), \qquad
\mathbb{E}_\xi[\|g(\theta; \xi) - \mathbb{E}[g(\theta; \xi)]\|^2] \le \sigma^2 < \infty
\end{equation}
for all $\theta$ in the iterate sequence. 
The uniformity of this bound is critical as token switching events (where the identity of the maximizing token $\bar y$ changes abruptly) can cause large variations in the subgradient. 
The uniform bound $\sigma^2$ ensures that even at these nondifferentiable switching boundaries, the stochastic oracle's noise remains bounded, preventing variance spikes that could destabilize convergence.
Furthermore, we assume the iterates $\{\theta_k\}$ remain in a compact set $\Theta$ where the local Lipschitz constant $L_{\mathcal{L}}$ is uniform.

We assume that for all $\theta \in \Theta$, the Clarke subgradients are uniformly bounded:
\begin{equation}
\sup_{g \in \partial \mathcal{L}_{\mathrm{Cal}}(\theta)} \|g\| \le G < \infty.
\end{equation}
This holds because $\mathcal{L}_{\mathrm{Cal}}$ is Lipschitz on the compact set $\Theta$.

The loss $\mathcal{L}_{\mathrm{Cal}}$ is \emph{Clarke regular} (chain rule holds for subgradients) because it is a composition of definable functions \cite{davis2020stochastic}. This ensures that the stochastic subgradient $g(\theta; \xi)$ is a valid element of the Clarke subdifferential in expectation.

The per-token calibration loss is non-convex and non-smooth, yet satisfies the regularity conditions required for stochastic subgradient methods: it is locally Lipschitz, differentiable almost everywhere, and admits a well-defined Clarke subdifferential $\partial \mathcal{L}_{\mathrm{Cal}}$. 
Since we assume the parameter set $\Theta$ is compact (as required for the uniform Lipschitz constant $L_{\mathcal{L}}$ in Section~\ref{adx:appendix-convergence-properties}), the logits $z_i(\theta; x)$ are bounded for any fixed input $x$. 
Consequently, the range of $\exp$ is restricted to a compact set, making softmax and sigmoid globally subanalytic on $\Theta$. 
Hence, $\mathcal{L}_{\mathrm{Cal}}$ is a composition of definable functions (neural network primitives, $\max$, and sigmoid), it satisfies the Kurdyka--{\L}ojasiewicz (KL) property \cite{bolte2014proximal}. 

Under the assumptions above—including Clarke regularity, bounded subgradients, and uniform Lipschitzness—and standard diminishing step-size conditions $\sum_k \eta_k = \infty$, $\sum_k \eta_k^2 < \infty$, stochastic subgradient descent converges almost surely to a Clarke stationary point of $\mathcal{L}_{\mathrm{Cal}}$ \cite{davis2020stochastic}. 
Importantly, token switching induces only transient changes in the active subgradient—the bounded variance and definability assumptions prevent unbounded noise, ensuring stable convergence.

\paragraph{Surrogate Approximation Gap}
Convergence guarantees apply to the stationary points of the \emph{surrogate} loss using $\widetilde{z}(x_t)$. 
The quality of these points relative to the true calibration optimum depends on the approximation error $|\widetilde{z}(x_t) - z(x_t)|$. 
This error is bounded by the margin between the top two predictions: as $\pi_\theta(y^\star \mid x_t) \to 1$, we have $\widetilde{z}(x_t) \to 1$ and the gap vanishes. 
Analysis of asymptotic calibration performance must account for this gap separately.

\paragraph{Uniform Lipschitz Constant}
The compact set $\Theta$ must be invariant under the optimization dynamics. In practice, this is ensured by weight clipping, explicit projection, or regularization that keeps parameters bounded. Since $\mathcal{L}_{\mathrm{Cal}}$ is definable, it is Lipschitz on any bounded set, guaranteeing a uniform constant $L_{\mathcal{L}}$.

\subsubsection{Proof of Proposition~\ref{prop:ordering-stability}}
\label{adx:ordering-stability-proof}
Consider the DPO margin
\(
\Delta_{\mathrm{DPO}}(x)
=
r_\theta(x,y^+) - r_\theta(x,y^-)
\).
Adding the calibration term induces a perturbation
\[
\Delta_{\mathrm{CAL}}(x)
=
\lambda \Bigl(
\frac{\partial L_{\mathrm{Cal}}}{\partial \log \pi_\theta(y^+\mid x)}
-
\frac{\partial L_{\mathrm{Cal}}}{\partial \log \pi_\theta(y^-\mid x)}
\Bigr).
\]

By assumption,
\(
\bigl|\partial L_{\mathrm{Cal}} / \partial \log \pi_\theta(y\mid x)\bigr|
\le G
\)
for all \((x,y)\), hence
\[
|\Delta_{\mathrm{CAL}}(x)| \le 2\lambda G.
\]

If \(2\lambda G < \Delta_{\min}\), then
\[
\Delta_{\mathrm{DPO}}(x) + \Delta_{\mathrm{CAL}}(x) > 0,
\]
so the sign of the preference margin cannot flip.
Thus the augmented objective preserves the original DPO ordering.
\hfill\(\square\)

\subsection{Analysis of Calibration Metrics}

In this section, we analyze the relationship between our training objective and standard calibration metrics.
Section~\ref{adx:appendix-ece-and-l1-risk}  demonstrates that minimizing the per-token $L_1$ surrogate risk minimizes a strict upper bound on the population Expected Calibration Error (ECE).
And Section~\ref{adx:appendix-properties-of-expected-calibration-error} situates marginal ECE within a hierarchy of metrics, analyzing its relationship to finer-grained variants like Classwise-ECE, while providing a comparision of the $L_1$ loss and Classwise-ECE as in Section~\ref{adx:appendix-ece-and-l1-risk}.

\subsubsection{Population ECE and L1 Risk}
\label{adx:appendix-ece-and-l1-risk}

In this section, we relate the per-sample $L_1$ calibration risk to the population Expected Calibration Error (ECE). 
Since ECE depends on the conditional expectation $\mathbb{E}[Z\mid C]$, it is difficult to optimize directly. 
We show that the $L_1$ risk upper-bounds ECE, making it a principled surrogate objective.
ECE captures the global discrepancy between the model's average confidence and its average accuracy, while the $L_1$ risk measures the per-sample (local) discrepancy. 

Let \(X\) be the input, \(C=c_\theta(X)\in[0,1]\) the model's reported confidence, and \(Z\in\{0,1\}\) the correctness indicator (\(Z=1\) if the model predicts correctly, 0 otherwise). 
The population ECE is
\begin{equation}\label{eq:ece-pop-repeat}
\mathrm{ECE}_{\mathrm{pop}}(\theta) = \mathbb{E}\big[\,\big|\mathbb{E}[Z\mid C] - C\big|\,\big].
\end{equation}
The per-example $L_1$ calibration risk is
\begin{equation}\label{eq:L1-pop-repeat}
\mathcal{R}_{\mathrm{L1}}(\theta) = \mathbb{E}\big[\,|C - Z|\,\big].
\end{equation}
We denote the model's true conditional accuracy as \(z=\mathbb{E}[Z\mid C=c]\).

Applying Jensen's inequality conditionally on \(C\) gives \(|\mathbb{E}[Z-C\mid C]| \le \mathbb{E}[|Z-C|\mid C]\), 
and taking the expectation over \(C\) yields
\begin{equation}\label{eq:ece-le-l1}
\mathrm{ECE}_{\mathrm{pop}}(\theta) \le \mathcal{R}_{\mathrm{L1}}(\theta).
\end{equation}

While $L_1$ risk upper-bounds ECE, they generally differ due to the conditional (irreducible) noise in predicting $Z$ from $C$. 
If $Z$ were deterministic given $C$, the two metrics would coincide. 
The following derivation explicitly characterizes this gap, showing that minimizing $L_1$ risk tightly controls the population ECE.

Let \(z=\mathbb{E}[Z\mid C=c]\) denote the model's true conditional accuracy given its reported confidence.  
Since \(Z\mid C=c \sim \mathrm{Bernoulli}(z)\), the conditional $L_1$ error is
\begin{align}
\mathbb{E}[|Z-C|\mid C=c]
 &= z\,|1-c| + (1-z)\,|0-c| \nonumber\\
 &= z(1-c) + (1-z)c
  = z+c - 2zc.
\end{align}
Comparing this to \(|z-c|\), we obtain a decomposition
\begin{equation}\label{eq:l1-decomp}
\mathbb{E}[|Z-C|\mid C=c]
  = |z-c| + 2\min\!\big\{\,z(1-c),\,c(1-z)\,\big\},
\end{equation}
which separates the $L_1$ error into the ECE term and a nonnegative conditional-noise term.
Taking expectation over \(C\) yields
\begin{equation}\label{eq:l1-ece-decomp}
\boxed{
\mathcal{R}_{\mathrm{L1}}(\theta)
 = \mathrm{ECE}_{\mathrm{pop}}(\theta)
   + 2\,\mathbb{E}\big[\min\{p(C)(1-C),\,C(1-p(C))\}\big]
}.
\end{equation}
The second term is nonnegative and represents the conditional noise of \(Z\) given \(C\); it is zero when \(Z\) is deterministic (\(p(C)\in\{0,1\}\)). 
Rewriting gives
\begin{equation}
\mathrm{ECE}_{\mathrm{pop}}(\theta)
 = \mathcal{R}_{\mathrm{L1}}(\theta)
   - 2\,\mathbb{E}\big[\min\{p(C)(1-C),\,C(1-p(C))\}\big]
 \le \mathcal{R}_{\mathrm{L1}}(\theta),
\end{equation}
recovering~\eqref{eq:ece-le-l1}.

Since \(\mathrm{ECE}_{\mathrm{pop}}(\theta) \le \mathcal{R}_{\mathrm{L1}}(\theta)\) for all \(\theta\), 
any (approximate) minimizer \(\hat\theta\) of the $L_1$ risk also achieves small population ECE:
\begin{equation}
\mathrm{ECE}_{\mathrm{pop}}(\hat\theta)
 \le \mathcal{R}_{\mathrm{L1}}(\hat\theta)
 \le \inf_{\theta}\mathcal{R}_{\mathrm{L1}}(\theta) + \varepsilon.
\end{equation}
If the conditional-noise term in~\eqref{eq:l1-ece-decomp} is small or roughly independent of \(\theta\), minimizing \(\mathcal{R}_{\mathrm{L1}}\) is effectively equivalent to minimizing the population ECE. 
Thus, \(\mathcal{R}_{\mathrm{L1}}\) is a tight, robust surrogate for calibration.

Using an $L_2$ calibration risk yields a strictly looser upper bound on ECE (via Cauchy--Schwarz), whereas the $L_1$ risk provides a tighter, geometry-matched bound that directly controls absolute calibration error.

\subsubsection{Properties of Expected Calibration Error}
\label{adx:appendix-properties-of-expected-calibration-error}
In this section, we study structural properties of Expected Calibration Error (ECE), with an emphasis on how ECE behaves under subpopulation restrictions.
In particular, we analyze how marginal ECE relates to confidence-weighted variants and to class-conditional calibration errors.
These results characterize what ECE measures—and what it may obscure—when calibration is evaluated on restricted subsets of the data.
The relationship between ECE and the population $L_1$ calibration risk $\mathcal{R}_{\text{Cal}}$ is formally developed in Section~\ref{adx:appendix-ece-and-l1-risk}, and is referenced here only to contextualize the hierarchy of calibration notions.

Throughout this section, we overload notation to facilitate classwise calibration analysis.
Let $\hat{Y} \in \{1,\dots,K\}$ denote the model’s predicted class, and let
$C \in [0,1]$ denote the model’s reported confidence associated with $\hat{Y}$.
We define the correctness indicator as
\[
Z \triangleq \mathbb{1}\{\hat{Y} = Y\}.
\]
Thus, $C$ should be interpreted as the predicted probability assigned to the selected class $\hat{Y}$, rather than the full probability vector.
Conditioning on $Y = k$ corresponds to restricting attention to samples whose true label is class $k$, as is standard in definitions of classwise calibration error.
Under this notation, the population $L_1$ calibration risk is given by
\[
\mathcal{R}_{\text{Cal}}(\theta) = \mathbb{E}\big[ |Z - C| \big].
\]

\begin{proposition}[ECE Dominates Bounded Reweightings]
\label{prop:ece_dominates}
Let $C \in [0,1]$ denote a model confidence score and $Z \in \{0,1\}$ the corresponding correctness indicator.
For any measurable weight function $w:[0,1]\to\mathbb{R}_+$ satisfying $\|w\|_\infty < \infty$, define the weighted calibration error
\[
\mathrm{ECE}_w
\;:=\;
\mathbb{E}\bigl[\, w(C)\, |\mathbb{E}[Z\mid C] - C| \,\bigr].
\]
Then
\[
\mathrm{ECE}_w
\;\le\;
\|w\|_\infty \cdot \mathrm{ECE}.
\]
\end{proposition}

\begin{proof}
Recall that the weighted calibration error is defined as
\[
\mathrm{ECE}_w
=
\mathbb{E}\!\left[
w(C)\,
\bigl|\mathbb{E}[Z \mid C] - C\bigr|
\right],
\]
where the weight function $w:[0,1]\to\mathbb{R}_+$ is assumed to be measurable and bounded.

Since $w(C)\ge 0$ almost surely, we may upper bound it by its essential supremum,
\[
w(C) \le \|w\|_\infty
\qquad \text{a.s.}
\]
Multiplying both sides by the nonnegative quantity
$\bigl|\mathbb{E}[Z \mid C] - C\bigr|$
yields
\[
w(C)\,
\bigl|\mathbb{E}[Z \mid C] - C\bigr|
\le
\|w\|_\infty\,
\bigl|\mathbb{E}[Z \mid C] - C\bigr|
\qquad \text{a.s.}
\]

Taking expectations on both sides and using the monotonicity of expectation,
\[
\mathbb{E}\!\left[
w(C)\,
\bigl|\mathbb{E}[Z \mid C] - C\bigr|
\right]
\le
\|w\|_\infty\,
\mathbb{E}\!\left[
\bigl|\mathbb{E}[Z \mid C] - C\bigr|
\right].
\]

Finally, observing that
\[
\mathbb{E}\!\left[
\bigl|\mathbb{E}[Z \mid C] - C\bigr|
\right]
=
\mathrm{ECE},
\]
we conclude that
\[
\mathrm{ECE}_w
\le
\|w\|_\infty \cdot \mathrm{ECE}.
\]
\end{proof}

Proposition~\ref{prop:ece_dominates} shows that standard ECE already controls
a broad family of confidence-weighted calibration errors.
We now turn to a different refinement of ECE obtained by conditioning on the true class label.
Rather than modifying ECE through confidence-dependent weights,
classwise calibration errors evaluate ECE on label-defined subpopulations, which is commonly used to assess heterogeneous calibration behavior across classes.

\begin{proposition}[ECE Lower-Bounds Classwise Calibration]
\label{prop:ece_cw}
Let $Y \in \{1,\dots,K\}$ denote the true class label.
Define the class-conditional calibration error
\[
\mathrm{ECE}_k
:=
\mathbb{E}\!\left[
\bigl|\mathbb{E}[Z\mid C,Y=k] - C\bigr|
\;\middle|\; Y=k
\right],
\]

and the class-frequency weighted CW-ECE
\[
\mathrm{CW\text{-}ECE}
:=
\sum_{k=1}^K
\mathbb{P}(Y=k)\,
\mathrm{ECE}_k
=
\sum_{k=1}^K
\mathbb{P}(Y=k)\,
\mathbb{E}\!\left[
\bigl|\mathbb{E}[Z\mid C,Y=k] - C\bigr|
\;\middle|\; Y=k
\right].
\]
Then
\[
\mathrm{ECE}
\;\le\;
\mathrm{CW\text{-}ECE}.
\]
\end{proposition}

\begin{proof}
By the law of total expectation,
\[
\mathbb{E}[Z \mid C]
=
\sum_{k=1}^K
\mathbb{P}(Y=k \mid C)\,
\mathbb{E}[Z \mid C, Y=k].
\]
Applying Jensen’s inequality to the convex function $|\cdot|$,
\[
|\mathbb{E}[Z\mid C] - C|
\le
\sum_{k=1}^K
\mathbb{P}(Y=k \mid C)\,
|\mathbb{E}[Z\mid C,Y=k] - C|.
\]
Taking expectations over the joint distribution and applying the tower property,
\[
\mathrm{ECE}
\le
\sum_{k=1}^K
\mathbb{E}\!\left[
\bigl|\mathbb{E}[Z\mid C,Y=k] - C\bigr|
\;\middle|\; Y=k
\right]
\mathbb{P}(Y=k)
=
\sum_{k=1}^K
\mathbb{P}(Y=k)\,
\mathrm{ECE}_k,
\]
where we used $\mathbb{E}[\mathbb{P}(Y=k\mid C) \cdot f(C)] = \mathbb{P}(Y=k)\,\mathbb{E}[f(C)\mid Y=k]$.
\end{proof}

Proposition~\ref{prop:ece_cw} formalizes a key limitation of marginal ECE, i.e. averaging over class labels can mask miscalibration that is present within individual classes.
This observation motivates calibration objectives that control classwise calibration error without explicitly conditioning on subpopulations.
While the construction and optimization of such objectives is deferred to the next section, we state below a dominance result that places classwise ECE within a broader calibration hierarchy.

\begin{proposition}[$\mathcal{R}_{\text{Cal}}$ Dominates Classwise-ECE]
\label{prop:rcal_dominance}
Let $Y \in \{1, \dots, K\}$ denote the class label, $C \in [0, 1]$ the model confidence, and $Z \in \{0, 1\}$ the correctness indicator. The Classwise Expected Calibration Error (CW-ECE) is defined as:
\begin{equation}
    \text{CW-ECE} = \sum_{k=1}^K P(Y=k) \, \mathbb{E}\big[ \left| \mathbb{E}[Z \mid C, Y=k] - C \right| \big].
\end{equation}
The Population $L_1$ Calibration Risk dominates CW-ECE:
\begin{equation}
    \text{CW-ECE} \le \mathcal{R}_{\text{Cal}}(\theta).
\end{equation}
\end{proposition}

\begin{proof}
Consider the inner term of the CW-ECE for a specific class $k$: $\left| \mathbb{E}[Z \mid C, Y=k] - C \right|$.
Since we are conditioning on $C$, the value $C$ is constant within the expectation. We can rewrite $C$ as $\mathbb{E}[C \mid C, Y=k]$. Substituting this back in:
\begin{equation}
    \left| \mathbb{E}[Z \mid C, Y=k] - \mathbb{E}[C \mid C, Y=k] \right| = \left| \mathbb{E}[Z - C \mid C, Y=k] \right|.
\end{equation}
Applying Jensen's inequality to the convex absolute value function $|\cdot|$:
\begin{equation}
    \left| \mathbb{E}[Z - C \mid C, Y=k] \right| \le \mathbb{E}\big[ |Z - C| \mid C, Y=k \big].
\end{equation}
Now, substituting this inequality back into the definition of CW-ECE:
\begin{align}
    \text{CW-ECE} &= \sum_{k=1}^K \mathbb{P}(Y=k) \, \mathbb{E}\!\left[ \bigl| \mathbb{E}[Z \mid C, Y=k] - C \bigr| \;\middle|\; Y=k \right] \\
    &\le \sum_{k=1}^K \mathbb{P}(Y=k) \, \mathbb{E}\!\left[ \mathbb{E}[ |Z - C| \mid C, Y=k ] \;\middle|\; Y=k \right].
\end{align}
By the law of total expectation (tower property), the inner nested expectations collapse:
\begin{equation}
    \mathbb{E}\!\left[ \mathbb{E}[ |Z - C| \mid C, Y=k ] \;\middle|\; Y=k \right] = \mathbb{E}\big[ |Z - C| \mid Y=k \big].
\end{equation}
Substituting this back into the sum:
\begin{equation}
    \text{CW-ECE} \le \sum_{k=1}^K P(Y=k) \, \mathbb{E}\big[ |Z - C| \mid Y=k \big].
\end{equation}
Finally, by the law of total expectation across all classes $Y$:
\begin{equation}
    \sum_{k=1}^K P(Y=k) \, \mathbb{E}\big[ |Z - C| \mid Y=k \big] = \mathbb{E}\big[ |Z - C| \big] = \mathcal{R}_{\text{Cal}}(\theta).
\end{equation}
Thus, minimizing the population $L_1$ risk $\mathcal{R}_{\text{Cal}}$ minimizes an upper bound on the Classwise-ECE.
\end{proof}

We emphasize that Proposition~\ref{prop:rcal_dominance} is not required for the analysis of ECE in this section.
Its role is to situate classwise ECE relative to surrogate calibration risks, which are formally introduced and analyzed in Section~\ref{adx:appendix-ece-and-l1-risk}. 

\end{document}